\documentclass{article}
\usepackage{comment}
\usepackage{subfigure}
\usepackage{graphics}
\usepackage[numbers,sort&compress]{natbib}

\usepackage[utf8]{inputenc} 
\usepackage[T1]{fontenc}    
\usepackage{xr-hyper}
\usepackage{hyperref}       
\usepackage{url}            
\usepackage{amsfonts}       
\usepackage{nicefrac}       
\usepackage{microtype}      
\usepackage{subfigure}
\usepackage{booktabs} 
\usepackage{amssymb,bm}
\usepackage{amstext, enumitem}
\usepackage{amsmath}
\usepackage{xspace}
\usepackage{theorem}
\usepackage{graphicx,color}
\usepackage{colordvi, wrapfig}
\usepackage{pifont}

\newtheorem{theorem}{Theorem}
\newtheorem{lemma}{Lemma}
\newtheorem{example}{Example}

\newtheorem{corollary}{Corollary}
\newtheorem{claim}{Claim}
\newtheorem{proposition}{Proposition}
\newtheorem{assumption}{Assumption}
\newtheorem{definition}{Definition}

\newenvironment{proof}{\par \smallskip{\bf Proof:}}{\hfill\stopproof}
\def\stopproof{\square}
\def\square{\vbox{\hrule height.2pt\hbox{\vrule width.2pt height5pt \kern5pt
\vrule width.2pt} \hrule height.2pt}}

\begin{document}

\title{ On the Benefit of Width for Neural Networks: Disappearance of Basins}

\author{
{Dawei Li} \thanks{Coordinated Science Laboratory, Department of ISE, University of Illinois at Urbana-Champaign, Urbana, IL. \texttt{dawei2@illinois.edu}.}
{\quad \quad \quad Tian Ding}\thanks{Department of Information Engineering, The Chinese University of Hong Kong, Hong Kong. \texttt{tianding@link.cuhk.edu.hk}.}
{\quad \quad  \quad Ruoyu Sun} \thanks{Coordinated Science Laboratory, Department of ISE, University of Illinois at Urbana-Champaign, Urbana, IL. \texttt{ruoyus@illinois.edu}.} \
\date{\today}
}

\newcommand{\fix}{\marginpar{FIX}}
\newcommand{\new}{\marginpar{NEW}}

\maketitle

\begin{abstract}
Wide networks are often believed to have a nice optimization landscape, but what rigorous results can we prove? To understand the benefit of width, it is important to identify the difference
between wide and narrow networks. In this work, we prove that from narrow to wide networks, there is a phase transition from having sub-optimal basins to no sub-optimal basins. Specifically, we prove two results: on the positive side, for any continuous activation functions, the loss surface of a class of wide networks has no sub-optimal basins, where ``basin'' is defined as the set-wise strict local minimum; on the negative side, for a large class of networks with width below a threshold, we construct strict local minima that are not global. These two results together show the phase transition from narrow to wide networks. 
\end{abstract}

\section{Introduction}
Current neural network models like ResNet \cite{he2016deep}
contain a very large number of parameters.
The large size is mainly because these networks are deep
(with many layers) and wide (with many neurons in one layer).
If we reduce the depth or the width, the performance of the models clearly deteriorates. The benefit of depth is believed to be stronger representation power (e.g. \cite{telgarsky2015representation}), but what is the benefit of width? 

One possible benefit is also the stronger representation power: perhaps a narrow network cannot fit the training data well, and thus performs worse than a wide network.
However, the findings of network pruning (e.g. \cite{han2015deep}) suggest 
that the disadvantage of narrow networks is largely due
to optimization issues. More specifically, the global optimal value of a small network problem can be quite good, but finding the global optimum by an optimization algorithm is hard.

The next question is: how does width help optimization? There is a recent trend of analyzing optimization algorithms for ultra-wide neural networks
\cite{jacot2018neural,allen2018convergence,zou2018stochastic,du2018gradient,mei2018mean,sirignano2018mean,rotskoff2018neural,sirignano2019mean,arajo2019meanfield,nguyen2019mean,dou2020training,wei2019regularization,fang2021modeling,chen2020generalized}. However, these results all require a much larger width than practice (e.g. infinite width \cite{jacot2018neural} or $O(N^{10})$ neurons \cite{du2018gradient} where $N$ is the number of samples). Further, these works mainly try to show  ``wide networks are good'', but not ``wide networks are better than narrow ones''. Therefore, they do not identify the specific benefit of width. 

As for the specific benefit, one popular intuition is that more parameters can ``smooth the landscape'' \cite{livni2014computational,lopez2018easing}.
 But what exactly does ``smoother landscape'' mean? 
One conjecture is that wide networks
have no sub-optimal local minima\footnote{In this paper, ``sub-optimal local minimum/basin'' or ``bad local minimum/basin'' refers to a local minimum/basin that is not global. Throughout this paper, we use these two terms interchangeably.} (e.g. \cite{lopez2018easing,nguyen2017loss2}).
This conjecture was stated as a theorem 
in a classical paper \cite{yu1995local}.
More specifically, \cite[Theorem 3]{yu1995local} claimed that for a 1-hidden layer network with more neurons than samples, no sub-optimal local minimum exists.
This result only requires very weak assumptions, and thus it seems to be strong evidence for the conjecture\footnote{However, a recent work \cite{ding2019spurious} disproved this result.}.
However, for decades none of the existing works have successfully extended the result of \cite{yu1995local} to deep networks. So far, there is no evidence that width can eliminate bad local minima for a general ``deep'' network.

In this work, we aim at understanding the specific benefit of width. From a scientific perspective, this benefit should satisfy two requirements. First, we hope to identify a property that holds for wide networks but not for narrow networks. 
Second, the theory should have minor assumptions on other aspects except for width, such as depth, activation function, and training data. In short, we endeavor to explore the {true benefit of width.}

\subsection{Main Contributions}
We first explain some geometrical concepts. 
A function is called  a ``global function'' if no sub-optimal local minimum exists. A function is called a ``weakly global function'' if it admits no sub-optimal set-wise strict local minima (also referred to as sub-optimal basins), defined in Definition \ref{def_2}; a graphical illustration is provided in Figure \ref{fig:weakglo}(a). Intuitively, weakly global functions may have flat regions of local minima (Figure \ref{fig:weakglo}(a)), but do not have bounded spurious regions surrounded by barriers (Figure \ref{fig:weakglo}(b)).

\begin{figure}
    \centering
    \subfigure[No bad basin]{\label{fig:weakglo:side:a}
    \includegraphics[width=1.8in, height=2.5cm]{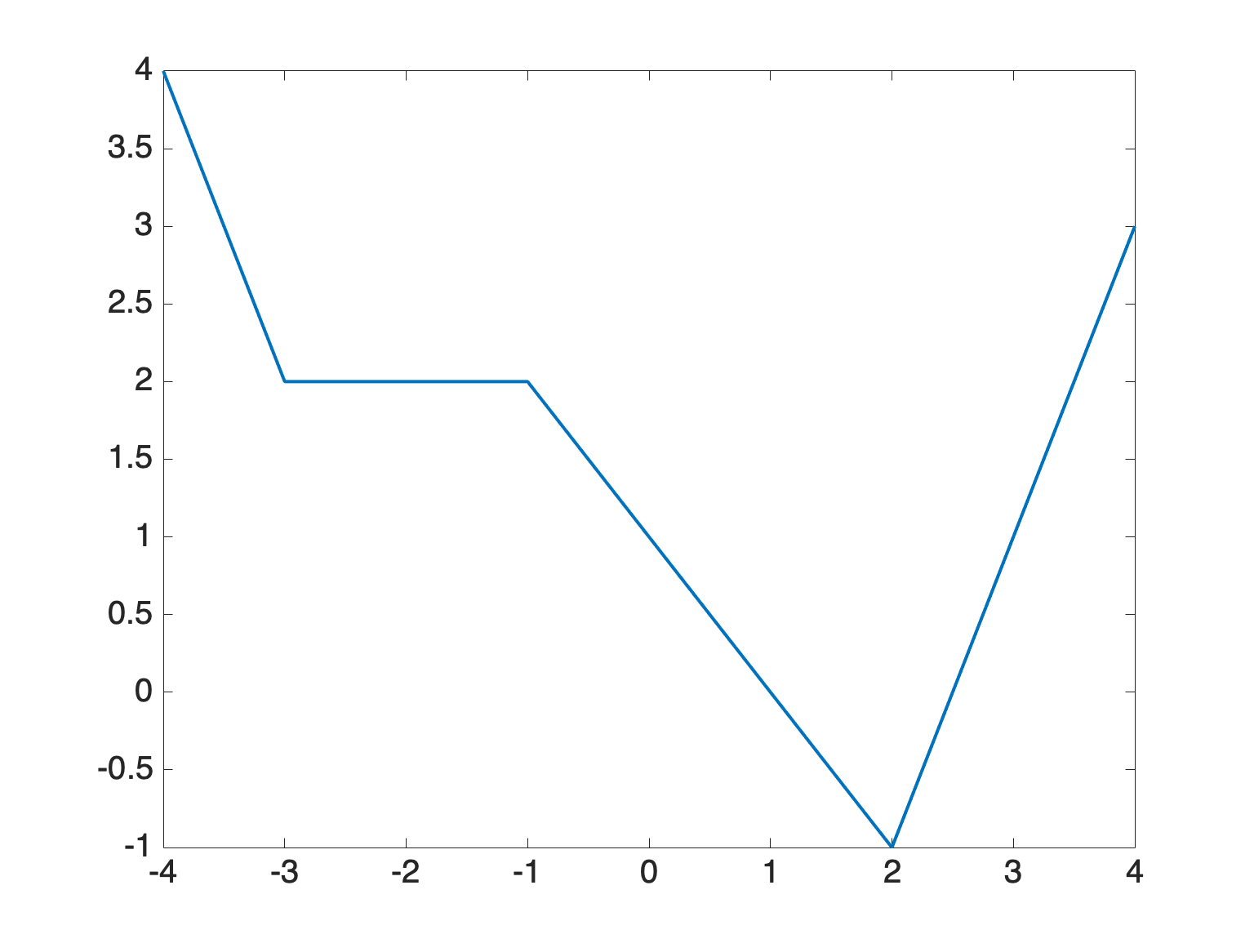}}
    \quad
    \subfigure[Has bad basin]{\label{fig:weakglo:side:b}
    \includegraphics[width=1.8in,height=2.5cm]{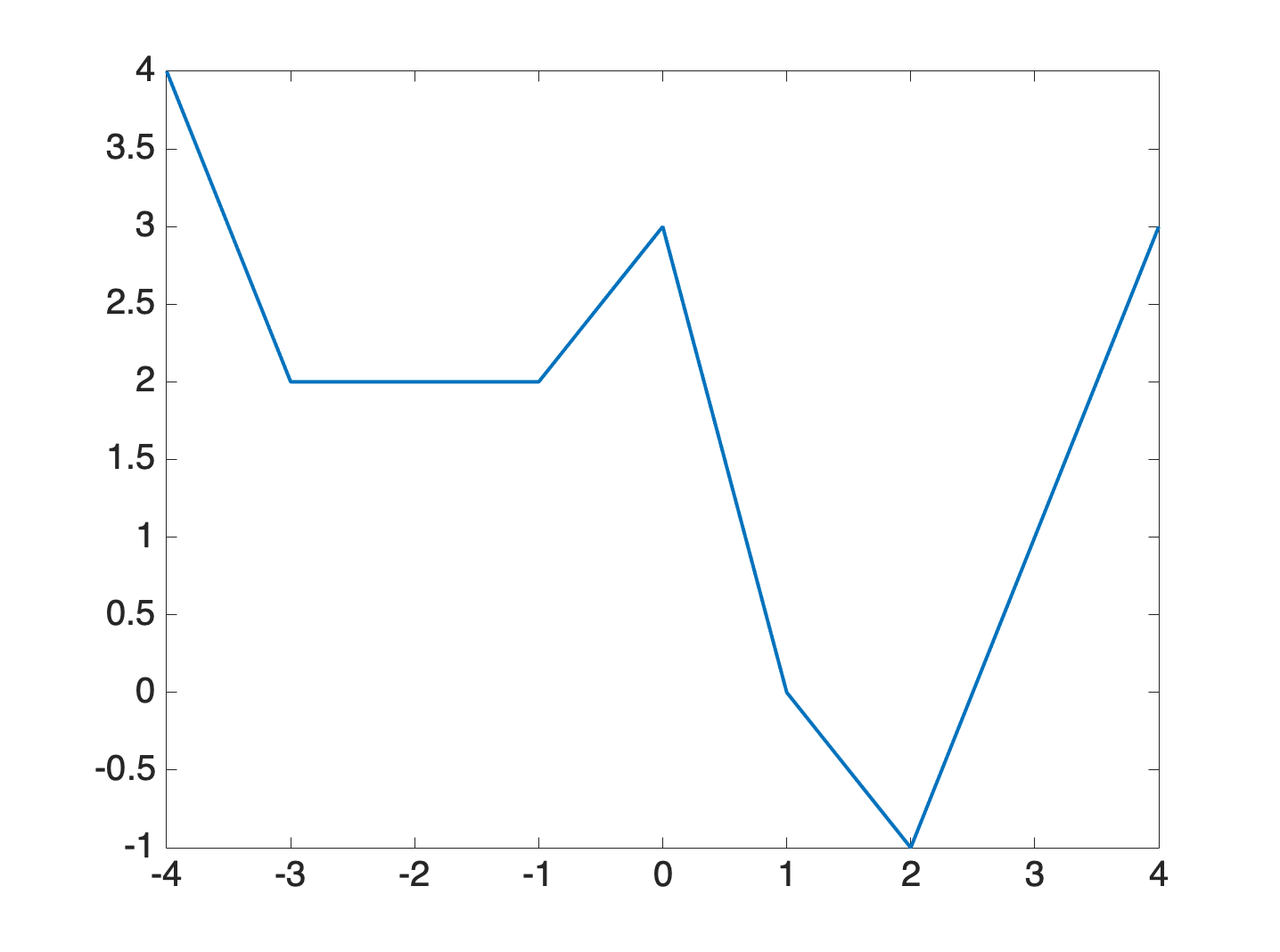}}
    \caption{An example of a weakly global function \subref{fig:weakglo:side:a} and a non-weakly-global function \subref{fig:weakglo:side:b}. Both functions have bad non-strict local minima, consisting of a plateau of $(-3, -1)$. The plateau in \subref{fig:weakglo:side:b} is surrounded by ``barriers'', entailing a sub-optimal set-wise strict local minimum in the sense of sets.}
    \label{fig:weakglo}
\end{figure}

Our first major result roughly states that for any continuous activation function, the loss function of wide networks is weakly global. Our second major result states that there exist some narrow networks that can have bad basins, i.e, are not weakly global. 
Therefore, from narrow to wide networks, there is a phase transition from ``having bad basins'' to ``no bad basin''. 

We provide a brief summary of our main results as follows.
\begin{itemize}
\item We study a geometrical property called ``Property PT'': at any point, after a small perturbation, there is a strictly decreasing path to a global minimum. We point out that Property PT implies ``weakly global''.
 
\item 
We prove that for a fully connected network with any number of layers, where the last hidden layer has at least $N$ neurons ($N$ is the number of data samples), and for any continuous activations, the loss function is weakly global (Theorem \ref{thm::continuous}). This is obtained by approximating the loss function by a class of neural networks with Property PT.

\item As an extension of the first result, we prove that
when there is one wide layer (not necessarily the last hidden layer) with at least $N$ neurons \footnote{This setting of network structure is more general than the first setting, and is quite close to practice: for instance, ImageNet dataset in ILSVRC competition has around 1.2 million training samples, while VGG and Inception network have 3 and 1.3 million neurons in the widest layer respectively. }, under extra assumptions on the structure and activation functions, the empirical loss is a weakly global function (Theorem \ref{thm::continuous-pyramid}).

\item We prove that 
for a large class of activations,
there exist 1-hidden-layer neural networks with less than $N$ neurons that have bad strict local minima (Theorem \ref{thm::bad_basin_exists}). This is in sharp contrast to the first result that
no bad basin exists if there are at least $N$ neurons. 
\end{itemize}

\subsection{Paper Organization and Notation}
This paper is organized as follows. 
We first review some related works in Section \ref{sec: related works}.
In Section~\ref{sec::main-results}, we present our main results. An overview of the main proof ideas is provided in Section~\ref{sec::proof-idea}. 
The formal proofs of the positive and negative results are respectively presented in Section \ref{sec::proof_positive} and \ref{sec::proof_negative}. Conclusions are presented in Section~\ref{sec::conclusions}.

Throughout this paper, we use lowercase letters (e.g., $a$), bold lowercase letters (e.g., $\mathbf{a}$), and uppercase letters (e.g., $A$) to denote scalars, column vectors, and matrices, respectively. $\mathbf{0}$ and $\mathbf{1}$ denote vectors or matrices with all-zero and all-one entries, respectively. For a column vector $\mathbf{a} \in \mathbb{R}^d$, we denote its $i$-th entry by $a_i$. For a matrix $A \in \mathbb{R}^{d_1 \times d_2}$, the entry in the $i$-th row and the $j$-th column is denoted by $(A)_{i,j}$ or $a_{i,j}$. We use $\circ$ to represent the Hadamard product. $\mathbb{N}$ denotes the set of all non-negative integers, and $\mathbb{N}_+$ denotes the set of all positive integers.

\section{Related Works and Discussions}\label{sec: related works}
$\quad \text{}$

\textbf{Practical motivation and implication}.
While our paper is mainly theoretical and has
a clear theoretical motivation, it is also strongly motivated by practical needs. Recently, there is an increasing demand for
training small neural networks, which is known
to be much more difficult than training large networks. 
Nevertheless, it is still not clear what makes the difference between small and big networks. 
Our work can be viewed as one step along this direction. 
In particular, our work suggests that
the training difficulty of small networks 
is (at least partially) due to bad basins, 
and wide networks do not suffer from bad basins. 
A recent work \cite{evci2019difficulty} 
empirically showed that for training small networks,
there are barriers between  two solutions obtained by different training methods, which can be viewed as empirical evidence that basins exist in small networks. See more discussions in Appendix \ref{appen::discussions}.

\textbf{Theoretical impact}.
Our work proves a nice geometrical property for overparameterized deep neural networks. Although it is established for supervised learning, the proof technique can be extended to other cases. For instance, a recent work for generative adversarial networks (GAN) \cite{sun2020towards} proved that the RpGAN loss has no bad basin in the function space, and then combined the function space result with the analysis of the current paper to show that the RpGAN loss in the parameter space also has no bad basin. This example demonstrates that the current work can serve as a foundation for landscape analysis of other problems that involve neural networks. 

\textbf{Results on shallow networks.}
The optimization landscape for neural network training problems has been extensively studied in recent years  \cite{ge2017learning,andoni2014learning,sedghi2014provable,janzamin2015beating,haeffele2015global,gautier2016globally,brutzkus2017globally,soltanolkotabi2017learning,soudry2017exponentially,goel2017learning,boob2017theoretical,du2017convolutional,zhong2017recovery,li2017convergence,liang2018understanding,du2018power,mei2018mean,sirignano2018mean,safran2017spurious,wang2019learning,chizat2018global,li2018learning, de2020sparsity, lacotte2020all, arjevani2021analytic, tian2017analytical}. These works only study one-hidden-layer networks, and many have strong assumptions on data distributions or the activation functions. In addition, none of them tries to differentiate narrow and wide networks.

\textbf{Results on deep linear networks.}
There are a few works showing that
for deep linear networks no sub-optimal local-min exists \cite{baldi1989neural,kawaguchi2016deep,freeman2016topology,hardt2016identity,yun2017global,laurent2018deep}.
However, linear neurons are not used in practice.
In addition, these works do not show the effect
of width on the landscape, as all linear networks
are equally good. 

\textbf{Algorithmic analysis.}

Some recent works  \cite{jacot2018neural,du2018gradientb,allen2018convergence,zou2018stochastic,liu2020loss,nguyen2021proof, huang2020dynamics, oymak2020toward} analyze the convergence of gradient descent (GD) to global minima. Their conclusions are stronger than our landscape analysis since they can prove linear convergence, but their assumptions are also stronger. In particular, these works require a large number of neurons at each layer, e.g. $O(N^{2})$ neurons, where $N$ is the number of samples. In practical neural networks such as VGG and ResNet, the number of neurons is in the order of $O(N)$, and even $O(N^2)$ neurons are too many to use in practice.
Recently, some works started to study the first-order methods for mildly over-parameterized neural networks. 
\cite{ji2019}\cite{chen2019much} showed that the width can be smaller than $O(N)$ if the dataset is separable with a large enough margin. However, if the dataset contains points that are close to each other with different labels then the number of parameters required could be very large or infinity. Note that all these works on algorithmic analysis only analyzed wide networks, and did not intend to show the difference between narrow and wide networks. 

\textbf{Landscape analysis of mildly wide deep networks.}
A few works \cite{nguyen2017loss1,nguyen2017loss2,shamir2018resnets,nguyen2018loss2} also studied mildly wide neural networks from the perspective of loss landscape, but they either only showed a weaker conclusion that a subset of local minima are global minima, or require stronger assumptions that the activation functions need to be smooth and strictly increasing (thus do not cover popular activation functions like ReLU, leaky ReLU and Swish \cite{ramachandran2018searching}, as discussed later). \cite{cooper2021global} also characterizes the geometric properties of critical points of wide neural networks.

We note that a recent work \cite{venturi2018spurious} provided a similar result to ours for general activation functions. However, it only studied 1-hidden-layer networks, while we consider an arbitrary network depth. Further, we believe that it is not easy to extend the proof technique in \cite{venturi2018spurious} to the multi-layer case.

\textbf{Various activation functions.}
Most existing theoretical works only analyze one activation function or one special class of activation functions (with the exception 
\cite{venturi2018spurious} which is limited to 1-hidden-layer networks).
To show the diversity of the activation functions considered in the literature, below we list some popular ones. 
(a) Sigmoid activation $\text{sigmoid}(t) = 1/(1 + e^{-t})$ and tanh activation  
$\text{tanh}(t) = (1 - e^{-2t})/(1 + e^{-2t})$.
They are strictly increasing but non-convex. 
(b)  ReLU activation 
$ \text{ReLU}(t) = \max \{ t , 0 \}.$ It is non-decreasing and convex, but non-smooth and not strictly increasing. 
(c) Leaky ReLU. It is strictly increasing but non-smooth.
(d) Polynomial activations (studied in, e.g.,
\cite{kileel2019expressive, soltanolkotabi2018theoretical, du2018power}). They are smooth but not monotone. In particular, a number of works \cite{soltanolkotabi2018theoretical,du2018power} consider the quadratic activation.
(e) The Swish activation $\text{swish}(t) = t \cdot  \text{sigmoid}(t)$,
which is recommended by \cite{ramachandran2018searching} as one of
the best choices for image classification after searching for a large class of functions. It is smooth and non-monotone. 
(f) GeLU activation \cite{hendrycks2016gaussian} which has been increasingly popular recently in natural language processing (e.g., \cite{liu2019roberta}). It is smooth and non-monotone. 
(g) sine activation \cite{sitzmann2020implicit}
which was found to be much better than ReLU for solving problems involving physical signals. It is smooth and non-monotone. 
We notice that these activation functions exhibit quite different properties, and we are not aware of any previous result (other than the very recent work \cite{venturi2018spurious} discussed before) on the landscape study that can cover even a half of them. In contrast, our positive result Theorem \ref{thm::continuous} applies to all the activation functions mentioned above.

\textbf{Generalization error}.
Finally, landscape analysis is just one part 
of machine learning theory, which includes representation, optimization and generalization. In terms of the generalization error, many recent works \cite{neyshabur2017exploring,bartlett2017spectrally,poggio2018theory,belkin2018reconciling,mei2019generalization} try to understand why over-parameterization does not cause overfitting. This is a line of research that is orthogonal to optimization studies. 

\section{Settings and Main Theorems}\label{sec::main-results}
In this section, we present our problem settings and main results. There are two sets of results: the positive results on the absence of sub-optimal basins for wide networks, and the negative results on the existence of sub-optimal basins for narrow networks. The formal proofs of the positive and negative results will be presented in Section \ref{sec::proof_positive} and Section \ref{sec::proof_negative}, respectively.

\subsection{Positive Results}
Consider a fully connected neural network with $H$ hidden layers. For an input sample $\mathbf{x}\in\mathbb{R}^{d_0}$, the network output is denoted by
\begin{equation}
\mathbf{t}_{H+1}(\mathbf{x})=W_{H+1}\sigma_H(W_H\cdots W_2\sigma_1(W_1\mathbf{x}))\in\mathbb{R}^{d_{H+1}}
\end{equation}
where for the $h$-th hidden layer, $d_h$ is the number of neurons, $W_h\in\mathbb{R}^{d_{h}\times d_{h-1}}$ is the weight matrix, and $\sigma_{h}$ is the activation function.
Assume that there are $N$ training samples. Let $X \triangleq [\mathbf{x}^{(1)}, \cdots, \mathbf{x}^{(N)}]\in\mathbb{R}^{d_0\times N}$ and $Y \triangleq [\mathbf{y}^{(1)},  \cdots, \mathbf{y}^{(N)}]\in\mathbb{R}^{d_{H+1}\times N}$
be the input and output data matrices. We can thus represent the network output as a matrix $T_{H+1} \in \mathbb{R}^{d_{H+1}\times N}$, given by the following recursive form:
\begin{subequations}
\label{eq 1}
\begin{align}
&T_1  = \sigma_1(W_1X),\\
&T_h  =\sigma_h(W_hT_{h-1}), \quad h=2, 3, \cdots, H,\\
&T_{H+1}  = W_{H+1}T_H
\end{align}
\end{subequations}
where $T_h\in\mathbb{R}^{d_h\times N}$ denotes the output of the $h$-th hidden layer.\footnote{In this paper, we do not consider bias in the network model to simplify the analysis. In fact, adding bias is equivalent to adding an all-one row to the input matrix of each layer. This does not affect the assumptions and thus does not change the analysis.}

Let $W=(W_1, \cdots, W_{H+1})$ denote all the weights. Define the empirical loss as
\begin{equation}\label{equation: loss function}
E(W) =l(Y, T_{H+1})=l(Y, W_{H+1}T_H)
\end{equation}
where $l$ is the loss function. Then, the training problem of the considered network is to find a $W$ to minimize the empirical loss $E(W)$.

By the positive results, we aim to characterize the benefit of network width. To establish our first theorem, we introduce an important landscape property termed ``Property PT''.
\begin{definition}
\label{def::path_to_globalinf}
For a continuous function $f:\mathbb{R}^m \rightarrow \mathbb{R}$, a (strictly) decreasing path from a given point $\mathbf{p}_0 \in \mathbb{R}^m$ to a global infimum is defined to be any continuous mapping $g:[0,1) \rightarrow \mathbb{R}^m$ such that $g(0) = \mathbf{p}_0$, $\lim_{t\rightarrow 1} f(g(t)) = \inf_{\mathbf{p} \in \mathbb{R}^m}f(\mathbf{p})$, and that $f(g(t))$ is (strictly) decreasing for $t\in[0,1)$.
\end{definition}
\begin{definition}
	We say a function $f:\mathbb{R}^m\rightarrow\mathbb{R}$ satisfies Property PT if starting from any point $\mathbf{p}\in\mathbb{R}^m$ that is not a global minimum, there exists an arbitrarily small perturbation such that from the perturbed point $\hat{\mathbf{p}}$, there exists a strictly decreasing path to the global infimum of $f$.
\end{definition}
Roughly speaking, Property PT identifies a class of loss surfaces on which the global infimum is ``approachable'' by a strictly decreasing path from \textit{almost} everywhere (but not everywhere). Note that a function with Property PT may still possess sub-optimal local minima. For example, the function $f(u, v)=(1-uv^2)^2$ has a non-strict local minimum $(-1, 0)$, since any point $(u', v')$ in a sufficiently small neighborhood of $(-1, 0)$ satisfies $uv^2\leq0$, leading to $f(u',v')>f(-1,0)=1$. Further, this local minimum is sub-optimal as $f(1,1) = 0<f(-1,0)$. However, $f$ satisfies Property PT: starting from any initial point $(u, v)$ that is not a global minimum, there exists an arbitrarily small perturbation $(u', v')$ that satisfies $v'\neq 0$; then we can construct a strictly decreasing path by connecting $(u', v')$ and a global minimum $(1/(v')^2, v')$.

Then, we specify several assumptions on the training dataset, the loss functions, the width, and the activation functions.
\begin{assumption}
\label{ass::overall}
\quad\par
	\begin{enumerate}[label=A\arabic*]
	    \item\label{ass::data} $\mathbf{x}^{(n)} \not= \mathbf{0}$ and $\mathbf{x}^{(n)}\neq \mathbf{x}^{(n')}$, for all $n, n' = 1,2,\cdots, N$ with $n \not= n'$;
		\item\label{ass::loss} The loss function $l(Y, \hat{Y})$ is convex with respect to $\hat{Y}$;
		\item\label{ass::overpar} The width of the last hidden layer satisfies $d_H\geq N$;
		\item\label{ass::neuron} The activation function $\sigma_h$ is continuous for all $h = 1,\cdots, H$.
	\end{enumerate}	
\end{assumption}
Assumption \ref{ass::data} requires the input data samples to be non-zero and to be distinguished from each other. Assumption \ref{ass::loss} holds for almost all commonly-used loss functions including quadratic loss, hinge loss, etc. Assumption \ref{ass::overpar} is the
assumption of ``wide network'', i.e.,
the last hidden layer has no fewer neurons than
the number of samples.
 There is no assumption on the width of other hidden layers.
Some works call this assumption ``over-parameterization''.  Assumption \ref{ass::neuron} is very mild as it only requires the activations to be continuous.
\begin{assumption}[Special Activation Functions]
\label{ass::activation}
For each $1 \leq h \leq H$, the activation function $\sigma_h$ is analytic, and there exist $N$ distinct non-negative integers $l_0, l_1, \cdots, l_{N-1} \in \mathbb{N}$, which form an arithmetic sequence, such that 
\begin{equation}
\sigma_h^{(l_0)}(0), \sigma_h^{(l_1)}(0), \cdots, \sigma_h^{(l_{N-1})}(0) \not = 0
\end{equation}
where $\sigma_h^{(l)}(0)$ denotes the $l$-th order derivative of $\sigma_h$ at zero.
\end{assumption}
Assumption \ref{ass::activation} identifies a class of analytic functions that has $N$ non-vanishing derivatives at zero, whose orders form an arithmetic sequence.\footnote{In the first version of this paper, we require that the first $N$ orders of derivatives, i.e., $\sigma_h(0), \sigma_h'(0),\cdots, \sigma_h^{(N-1)}(0)$, are non-zero. This is a special case of Assumption \ref{ass::activation}.} This includes many commonly-used activation functions such as sigmoid and softplus, but it does not cover ReLU since it requires smoothness (as mentioned before, ReLU can be covered by using the approximation trick). Now, we present our first main theorem, which states that for generic data and the above special class of activation functions, the empirical loss of a wide neural network has Property PT.
\begin{theorem}\label{thm::deep-convex-analytic}
Suppose that a fully connected network satisfies Assumption \ref{ass::overall} and Assumption \ref{ass::activation}. Then, the empirical loss function $E(W)$ defined in \eqref{equation: loss function} satisfies Property PT.
\end{theorem}

In what follows, we establish positive results showing that no sub-optimal basins exist for wide networks. Here, we use the term ``basin'' to denote ``set-wise strict local minimum'' (which is compact by definition), a notion borrowed from \cite{josz2018theory} and formally presented below.
\begin{definition}[Set-wise strict local minimum, ``basin'']
\label{def_1}
We say that a compact subset $X \subseteq S$ is a strict local minimum of $f:S \rightarrow \mathbb{R}$ in the sense of sets if there exists $\varepsilon>0$ such that for all $x\in X$ and all $y\in S\setminus X$ satisfying $\min_{x'\in X}\|y-x'\|_2\leq\varepsilon$, it holds that $f(x)< f(y)$.
\end{definition}
Definition \ref{def_1} generalizes the notion of strict local minimum from the sense of points to the sense of sets. Note that any point-wise strict local minimum must be a set-wise strict local minimum, but not vice versa.
For example, the plateau of $[-3, -1]$ in Figure \ref{fig:weakglo}\subref{fig:weakglo:side:b} is a set-wise (but not point-wise) strict local minimum.

\begin{definition}[Weakly global function]
\label{def_2}
We say that $f:S\rightarrow\mathbb{R}$ is a weakly global function if it is continuous and every set-wise strict local minimum contains a (point-wise) global minimum of $f$. 
\end{definition}
Definition \ref{def_2} identifies an important class of continuous functions, termed weakly global functions, which admit no bad strict local minima in the sense of sets, and hence no sub-optimal basins. We note that although weakly global functions do not have sub-optimal set-wise strict local minima, they may still possess sub-optimal infimum stretching to infinity. For example\footnote{This example is provided by an anonymous reviewer of the submitted version. We sincerely thank this reviewer for the contribution.}, consider $f:\mathbb{R}\rightarrow \mathbb{R}$, $f(t)=\min(1+e^t,(1-t)^2)$. This is a weakly global function, but the function value decreasingly converges to 1 as $t\rightarrow -\infty$, which is greater than the global minimal value $f(1) = 0$.

It can be shown that Property PT implies the absence of sub-optimal basins.
\begin{proposition}\label{prop::PT_weaklyglobal}
A function $f:\mathbb{R}^m\rightarrow\mathbb{R}$ is weakly global if it has Property PT.
\end{proposition}
By Proposition \ref{prop::PT_weaklyglobal}, we immediately obtain that the empirical loss function in Theorem \ref{thm::deep-convex-analytic} is also weakly global.
\begin{corollary}\label{cor::deep-convex-analytic}
Suppose that a fully connected network satisfies Assumption \ref{ass::overall} and Assumption \ref{ass::activation}. Then, the empirical loss function $E(W)$ defined in \eqref{equation: loss function} is a weakly global function.
\end{corollary}
Then, we present our second main theorem, which extends Corollary \ref{cor::deep-convex-analytic} to all continuous activation functions.
\begin{theorem}\label{thm::continuous}
Suppose that a fully connected network satisfies Assumption \ref{ass::overall}. Then, the empirical loss function $E(W)$ defined in \eqref{equation: loss function} is a weakly global function. In other words, the loss function has no sub-optimal set-wise strict local minima.
\end{theorem}
Since a (point-wise) strict local minimum is a special case of a set-wise strict local minimum, we immediately have the following corollary. 
\begin{corollary}
\label{cor::continuous}
The empirical loss function for a fully connected neural network satisfying Assumption \ref{ass::overall}  has
no sub-optimal strict local minima in the sense of points. 
\end{corollary}

The notion of sub-optimal basin is closely related to the term ``spurious valley'', defined as a connected component of a sub-level set, within which the loss value cannot be made arbitrarily close to the global infimum. This is a more general notion because a spurious valley may be unbounded, while a sub-optimal basin is always a compact set. An interesting line of works \cite{venturi2018spurious,nguyen2018loss2,nguyen2019connected} has shown that under certain assumptions, no spurious valleys exist for wide neural networks. For example, \cite{venturi2018spurious} proved that no spurious valleys exist for 1-hidden-layer network with ``low intrinsic dimension''. \cite{nguyen2019connected} further proved that there are no spurious valleys for wide multi-layer neural networks with a class of strictly monotonic activations.

In this paper, we study sub-optimal set-wise strict local minima rather than spurious valleys. On one hand, our property is weaker since no spurious valleys implies no sub-optimal basins but the reverse is not true \cite{sun2020global}. On the other hand, this weaker property has two benefits: i) can cover a broader class of activation functions; ii) makes it easier to establish phase transition. More specifically, the existing results on ``no spurious valleys'' for deep networks post strong conditions on activation functions, such as strict monotonicity \cite{nguyen2019connected}, while our landscape property holds for all continuous activations.
Because of this breadth, we are able to build phase transitions for many activation functions (not only strictly monotone ones).

In Theorem \ref{thm::continuous}, the last hidden layer of the network needs to be sufficiently wide. Next, we show that if a pyramid structure is added after the ``wide layer'', the resulting loss surface still admits no sub-optimal basins. We specify such pyramid structure by the the following assumption. 
\begin{assumption}
\label{ass::pyramid-overall}
There exists $1 \leq h_0 \leq H$ such that
\begin{enumerate}[label=B\arabic*]
	\item\label{ass::pyramid-structure} $d_{h_0}\geq N$, $d_{h_0}\geq d_{h_0+1}\geq\cdots\geq d_{H+1}$;
    	\item\label{ass::pyramid-neuron} For all $h_0+1 \leq h \leq H$, the activation function $\sigma_h$ is non-increasing or non-decreasing over $\mathbb{R}$. 
\end{enumerate}	
\end{assumption}
It can be shown that our ``no sub-optimal basins'' result for over-parameterized networks still holds after adding the pyramid architecture:

\begin{theorem}\label{thm::continuous-pyramid}
Suppose that a fully connected neural network satisfies Assumption \ref{ass::data}, \ref{ass::loss}, \ref{ass::neuron}, and Assumption \ref{ass::pyramid-overall}. Then, the empirical loss function $E(W)$ defined in \eqref{equation: loss function} is a weakly global function. In other words, the loss function has no sub-optimal set-wise strict local minima.
\end{theorem}

\subsection{Negative Results}
The negative results mainly state that without the assumption on the width, the nice property of ``no sub-optimal basins'' may not hold. In this part, we consider a 1-hidden-layer network with unit input dimension and unit output dimension. Therefore, we adopt a simpler notation system for convenience of presentation. Let $m$ denote the width of the hidden layer, and $\mathbf{x} = (x^{(1)}, \cdots, x^{(N)})^\top \in \mathbb{R}^{N}$
, $\mathbf{y} = (y^{(1)}, \cdots, y^{(N)})^\top \in \mathbb{R}^{N}$ be the vectors collecting all input and output data samples, respectively. The network output is represented by $\hat{\mathbf{y}}^\top=\mathbf{v}^\top \sigma(\mathbf{wx}^\top)$ where $\mathbf{v} = (v_1, \cdots, v_{m})^\top\in \mathbb{R}^{m}$ and $\mathbf{w}=(w_1, \cdots, w_{m})^\top \in \mathbb{R}^{m}$ are the weights to the output layer and the hidden layer, respectively. We adopt the quadratic loss function $l(\mathbf{y}, \hat{\mathbf{y}})=\|\mathbf{y}-\hat{\mathbf{y}}\|_2^2$, and thus the empirical loss is given by
\begin{equation}
E(\mathbf{v}, \mathbf{w})=\|\mathbf{y} - \hat{\mathbf{y}}\|^2_2 = \sum_{n=1}^N\left[y^{(n)}-\mathbf{v}^\top \sigma(\mathbf{w}\cdot x^{(n)})\right]^2.
\end{equation}

The first negative result shows that under certain conditions, ``narrow'' ($m<N$) networks have sub-optimal basins.
\begin{theorem}
\label{thm::bad_basin_exists}
Consider a data vector $\mathbf{x}\in \mathbb{R}^N$, $N>1$. Suppose that $\mathbf{x}$ has no zero entries and no identical entries, i.e., $x^{(i)}\not= 0$ and $x^{(i)}\not= x^{(i')}$ for all $1\leq i,i'\leq N$, $i \not= i'$. Then, for any $m <N$, there exist output data $\mathbf{y}\in\mathbb{R}^N$ and a 1-hidden-layer network with $m$ neurons, such that the empirical loss has a sub-optimal basin.
\end{theorem}
Theorem \ref{thm::bad_basin_exists} shows that if the network width is smaller than $N$, for generic input data $\mathbf{x}$, we can always find a neural network such that a bad basin exists. Together with Theorem \ref{thm::continuous}, we characterize a phase transition from existence to non-existence of sub-optimal basins. However, given a specific activation function, this phase transition may not exist. In fact, it can be shown that for positive homogeneous activation functions (e.g., ReLU), there is always no set-wise strict local minima regardless of the network width.\footnote{Theorem \ref{thm::homogeneous_no_bad_basin} is presented by an anonymous reviewer of the submitted version. We sincerely thank this reviewer for the contribution.}
\begin{theorem}
\label{thm::homogeneous_no_bad_basin}
Consider a 1-hidden-layer neural network with activation function satisfying $\sigma(\alpha t)=\alpha \sigma(t)$ for any $t\in\mathbb{R}$ and $\alpha > 0$. For any width $m$ and any data $(\mathbf{x}, \mathbf{y})$, the empirical loss has no set-wise strict local minima.
\end{theorem}

Naturally, one may wonder what type of activation functions admit a phase transition on the existence of sub-optimal basins. The following theorem characterizes such a class of activation functions.
\begin{theorem}
\label{thm::bad_basin_smooth}
Consider a 1-hidden-layer neural network with $N\geq 3$, $m=1$ (single neuron). The activation function $\sigma$ is twice continuously differentiable with  $\sigma(0),\sigma'(0),\sigma''(0)\not=0$. Suppose that the input data $\mathbf{x}$ has no zero entries and no identical entries. Then, there exists output data $\mathbf{y}\in\mathbb{R}^N$ such that the empirical loss has a sub-optimal basin.
\end{theorem}
We see that for a large class of activation functions (which constitute a dense set in the space of continuous functions by Lemma \ref{lemma::continuous_activation}), narrow networks do have sub-optimal basins. Increasing the network width can smooth the landscape by eliminating all sub-optimal basins. We define the term ``transition width'' as the smallest network width that eliminates all sub-optimal basins for any depth and generic input data. Then, for activation functions specified in Theorem \ref{thm::bad_basin_smooth}, the transition width takes place between $1$ and $N$. Further, Theorem \ref{thm::bad_basin_smooth} requires $\sigma''(0)\not=0$, i.e., non-linear around zero. Below we prove a similar result for another class of activation functions that are piecewise linear.
\begin{theorem}
    \label{thm::bad_basin_piecewise}
    Consider a 1-hidden-layer neural network with $N\geq 3$. The activation function $\sigma$ is piecewise linear with $\sigma(0) = \sigma'(0) = 0$. Suppose that $2m+1 \leq N$. Then, for a positive measure of input data $\mathbf{x}\in\mathbb{R}^N$, there exists output data $\mathbf{y}\in\mathbb{R}^N$ such that the empirical loss has a sub-optimal basin.
\end{theorem}

Table \ref{tab::transition_width} summarizes our results for different classes of activations and the corresponding transition widths. We also include more discussions in Appendix \ref{appen::discussions} on the implication of ``width eliminates bad basins''.
\begin{table}
\label{tab::transition_width}
\centering
\begin{tabular}{ |c|c|  }
 \hline
 Activation & Transition width $m^*$\\
 \hline
 \hline
 $\sigma \in C^2$ and $\sigma(0), \sigma'(0), \sigma''(0) \not= 0$ & $1<m^*\leq N$\\
 \hline
 piecewise linear with $\sigma(0) = \sigma'(0) = 0$ & $(N-1)/2 < m^* \leq N$\\
 \hline
 positive homogeneous & no phase transition\\
 \hline
\end{tabular}
\caption{Activation classes and corresponding transition widths. The term ``transition width'' is defined as the smallest network width that eliminates all sub-optimal basins for any depth and generic input data.}
\end{table}

Although Table \ref{tab::transition_width} does not cover all continuous activation functions, our proof of Theorem \ref{thm::bad_basin_exists} suggests a general computational approach on constructing sub-optimal basins for a specific neural network. By this approach, a numerical example for Swish activation function is presented in Section \ref{sec::proof_negative}.

Finally, one might also wonder what the implication of ``width eliminating bad basins'' is. See discussions in Appendix \ref{appen::discussions}.

\section{Main Idea and Proof Sketch}\label{sec::proof-idea}
In this section, we provide an overview of the proof of our positive results.

\subsection{First Idea: Property PT Holds for Many Activations}
Consider a $1$-hidden-layer network with $d$ input dimensions, $1$ output dimension and $m$ hidden-layer neurons. The output of the network is given by $\hat{\mathbf{y}}^T = \mathbf{v}^\top\sigma(WX) \in \mathbb{R}^{1\times N}$, where $X\in \mathbb{R}^{d\times N}$ is the input data matrix consisting of $N$ samples, $W\in\mathbb{R}^{m\times d}$ is the weight matrix to the hidden layer, and $\mathbf{v}\in\mathbb{R}^{m}$ is the weight vector to the output layer. For simplicity, consider the quadratic loss function $l(\mathbf{y},\hat{\mathbf{y}}) = ||\mathbf{y}-\hat{\mathbf{y}}||^2_2$. The network is assumed to be wide, i.e., $m \geq N$.

Now we denote $Z=\sigma(WX)\in\mathbb{R}^{m\times N}$ as the output matrix of the hidden layer. The 
optimization problem associated with the considered network is
\begin{equation}
\label{eq::sl-optimize}
\min_{\mathbf{v},W}\left|\left|\mathbf{y}^\top-\mathbf{v}^\top Z \right|\right|_2^2 .
\end{equation}
A simple but important observation \cite{gori1992problem} is that the problem \eqref{eq::sl-optimize} (with fixed $W$) is convex with respect to $\mathbf{v}$. Then, we can find a strictly decreasing path to the global minimum (with respect to $\mathbf{v}$) of \eqref{eq::sl-optimize}.\footnote{For a convex function, there exists a strictly decreasing path to the global infimum from any point that is not a global minimum. This result is formalized as Lemma \ref{lem::convex_strictlydecreasing}.} Moreover, if $Z$ is of full column rank, i.e., rank-$N$, then for any $\mathbf{y}$, there exists $\mathbf{v}$ such that $\mathbf{v}^\top Z=\mathbf{y}$, so the minimal value equals $0$, which is also the global minimal value of \eqref{eq::sl-optimize}. Thus, if for any network parameter $(\mathbf{v},W)$ that is not a global minimum, and the post-activation matrix $Z$ is full-rank, there is a decreasing path to a global minimum, and hence $(\mathbf{v}, W)$ cannot be a bad local minimum.

What if $Z$ is not full-rank? This is totally possible. For instance, if $W$ is a zero matrix, the corresponding $Z$ has at most rank-1. One natural solution is to use perturbation: after perturbing the parameter $(\mathbf{v},W)$, we may get a full-rank matrix $Z$ and thus a decreasing path to a global-min. If at any point $(\mathbf{v},W)$ this can be done, then the function satisfies property PT. We will formally state this derivation in a claim, and beforehand, we introduce another intermediate notion termed ``full-rank-reachable''. 

\begin{definition}[Full-rank-reachable]
\label{def::full_rank_reachable}
Consider a data matrix $X \in \mathbb{R}^{d \times N}$. If for any $W $, we can slightly perturb $W$ to $\hat{W}$ such that $ \hat{Z} = \sigma(  \hat{W} X)$ has rank $N$, then we say the activation $\sigma$ is full-rank-reachable. 
\end{definition}

\begin{claim}
  If $\sigma$ is full-rank-reachable,
  then the function defined in \eqref{eq::sl-optimize}
  satisfies Property PT. 
\end{claim}

The next crucial question is: what activation functions are full-rank-reachable (and thus guarantee Property PT)? A linear $\sigma $ is surely not enough, and it seems that certain non-linearity of $\sigma $ can provide lots of variations so that the matrix $\sigma(W X )$ can be perturbed to  have full rank. But how much ``non-linearity'' is needed? Is every non-linear function full-rank-reachable? The answer is ``no''. Specifically, we show that a monomial function $ \sigma(z) = z^p$ is not full-rank-reachable. 

\begin{example}[$z^p$ is not full-rank-reachable]
	Consider a 1-hidden-layer neural network with $2$ data samples, $1$ input dimension, $2$ hidden-layer neurons, and monomial activation $\sigma(z)=z^p$, $p > 0$. We have $W\in\mathbb{R}^{2\times 1}, X\in\mathbb{R}^{1\times 2}$, and therefore the rank of $WX$ is at most $1$. Denote $WX=\begin{bmatrix}
	a_1 & a_2\\ ta_1 & ta_2
	\end{bmatrix}$ where $t\in\mathbb{R}$. Then, $Z=\sigma(WX)=\begin{bmatrix}
	a_1^p & a_2^p\\ t^pa_1^p & t^pa_2^p
	\end{bmatrix}$ is always rank-deficient. This implies that we cannot find any $W$ such that $Z$ is of full column rank.
\end{example} 

This example implies that it is not straightforward to 
get full-rank-reachable, and the activation function needs to have ``sufficient nonlinearity''. In this paper, we identify a class of activation functions (analytic and with $N$ non-vanishing derivatives at zero whose orders form an arithmetic sequence, formally defined in Assumption \ref{ass::activation}) that are full-rank-reachable and thus guarantee
Property PT of the empirical loss. This yields Theorem \ref{thm::deep-convex-analytic}. To establish Theorem \ref{thm::continuous}, extra techniques are needed to extend this result to all continuous activation functions, which we discuss next.

\subsection{Extending to Any Continuous Activation}
We mentioned that there is a class of special analytic activation functions such that the empirical loss functions have Property PT. Unfortunately, these functions do not cover many commonly used activations like ReLU, leaky ReLU, etc.
How to proceed? 
The key observation is that these functions constitute a dense set in the space of continuous functions. That is, for any continuous function $f$, there exists a sequence of special analytic functions that uniformly converges to $f$. Based on this observation, a natural
idea is to show that the limit of a set
of ``nice'' functions is also ``nice'',
though we need to formally define ``nice''.
More specifically, we find that the property ``weakly global'' is preserved in the limit, then we obtain Theorem \ref{thm::continuous}.

To summarize, our proof of the positive results can be sketched in the following three steps:

\textbf{Step 1}: If the activation function is analytic and has non-vanishing derivatives at zero with an arithmetic sequence of $N$ orders, then Property PT holds, and the loss surface is a weakly global function.

\textbf{Step 2}: The set of activation functions in Step 1 is dense in the space of continuous functions.

\textbf{Step 3}: By the ``closedness'' of weakly global functions, the loss surface for any continuous activation is a weakly global function.

This proof sketch is illustrated in Figure  \ref{fig:roadmap}. 

\begin{figure}
    \centering
    \includegraphics[width=0.8\linewidth]{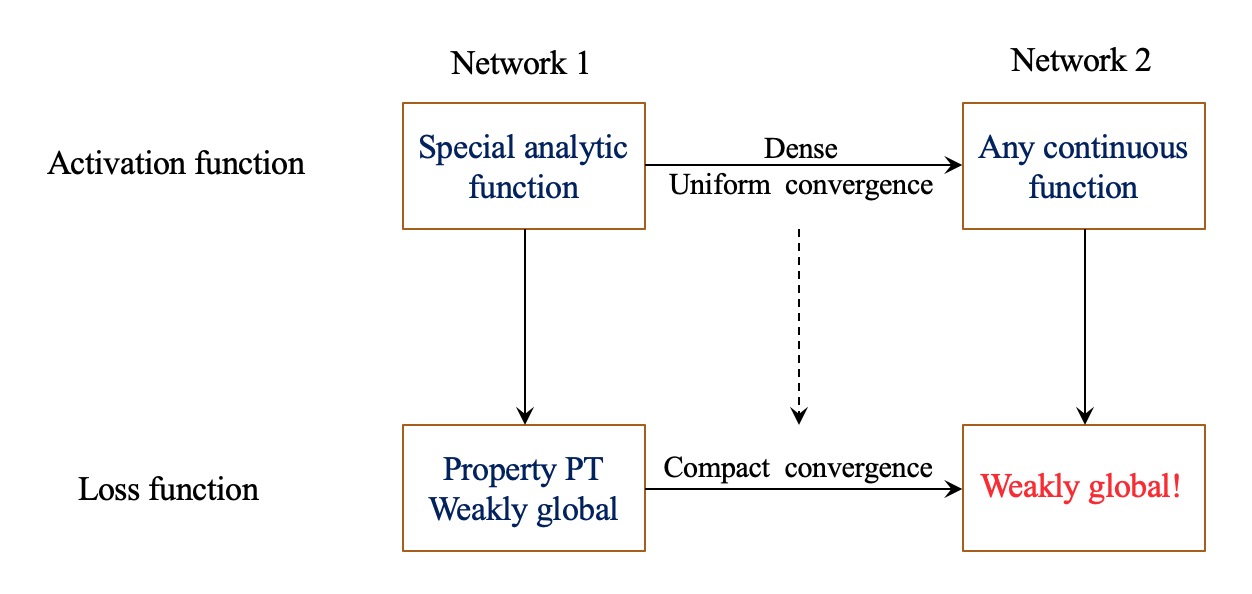}
    \caption{
    {\small Sketch of the proof. We first prove that the empirical loss function is weakly global for a class of analytic functions (the arrow from top left to bottom left). Then we show that these special functions are dense in the set of continuous functions (the arrow from top left to top right). Finally, we prove that the set of weakly global functions are closed (the arrow from the bottom left to bottom right). Combining the three steps, we obtain that the loss function is weakly global for any continuous activation (the arrow from top right to bottom right).}
    }
    \label{fig:roadmap}
\end{figure}

\section{Formal Proof of Positive Results}\label{sec::proof_positive}
As mentioned in Section \ref{sec::proof-idea}, our proof consists of three steps. By Step 1 we prove Theorem \ref{thm::deep-convex-analytic}, while by Step 2 and Step 3 we obtain Theorem \ref{thm::continuous}.

\subsection{Proof of Theorem \ref{thm::deep-convex-analytic} (Step 1)}
We first establish several lemmas and propositions needed for the proof. The first lemma shows that activation functions satisfying Assumption \ref{ass::activation} have a nice full-rank property.
\begin{lemma}\label{lem 1}
	Suppose that $\sigma$ is an analytic function satisfying Assumption \ref{ass::activation}. Then for any $N$ scalars $z_1, z_2, \cdots, z_N \in \mathbb{R}$ such that $z_n \not= 0$ and $z_n\neq z_{n'}$, for all $n, n'=1,2,\cdots, N$, $n\neq n'$, the following matrix
	\begin{equation}
	A=
	\begin{pmatrix}
	z_1^{l_0}\sigma^{(l_0)}(0) & z_2^{l_0}\sigma^{(l_0)}(0) & \cdots & z_N^{l_0}\sigma^{(l_0)}(0) \\
	z_1^{l_1} \sigma^{(l_1)}(0) & z_2^{l_1}\sigma^{(l_1)}(0) & \cdots & z_N^{l_1}\sigma^{(l_1)}(0) \\
	\vdots & \vdots & ~ & \vdots \\
	z_1^{l_{N-1}}\sigma^{(l_{N-1})}(0) & z_2^{l_{N-1}}\sigma^{(l_{N-1})}(0) & \cdots & z_N^{l_{N-1}}\sigma^{(l_{N-1})}(0) \\
	\end{pmatrix}
	\end{equation}
	is non-singular.
\end{lemma}
\begin{proof}
We first consider the matrix
\begin{equation}
	A^*=
	\begin{pmatrix}
	1 & 1 & \cdots & 1 \\
	z_1^{l_1-l_0} & z_2^{l_1-l_0} & \cdots & z_N^{l_1-l_0} \\
	\vdots & \vdots & ~ & \vdots \\
	z_1^{l_{N-1}-l_0} & z_2^{l_{N-1}-l_0} & \cdots & z_N^{l_{N-1}-l_0} 
	\end{pmatrix}\in \mathbb{R}^{N\times N}
\end{equation}
As $\{l_0, l_1, \cdots, l_{N-1}\}$ is an arithmetic sequence, $A^*$ is a Vandermonde matrix and hence is full-rank. For $A^*$, by first multiplying the $n$-th row by $\sigma^{(l_{n-1})}(0)$ for all $1\leq n \leq N$, then multiplying the $m$-th column by $z_m^{l_0}$ for all $1\leq m \leq N$, we obtain the matrix $A$. Since $\sigma^{(l_{n-1})}(0) \not= 0$ according to Assumption \ref{ass::activation}, and $z_m \not= 0$, $A$ is also a non-singular matrix.
\end{proof}

Next, we borrow an important result of \cite{mityagin2015zero} which states that the zero set of an analytic function is either the whole domain or of measure zero.
\begin{lemma} \label{lem 2}
For any $m\in \mathbb{Z}^+$, let $f: \mathbb{R}^m\rightarrow \mathbb{R}$ be a real analytic function. If $f$ is not identically zero, then its zero set $\Omega=\{\mathbf{z} \in \mathbb{R}^m \mid f(\mathbf{z})=0\}$ has measure zero.
\end{lemma}
Based on Lemma \ref{lem 2}, we have the following results.
\begin{lemma}\label{lem 3}
	Suppose that $\sigma$ is a non-constant analytic function.
	Given $\mathbf{a}, \mathbf{b} \in \mathbb{R}^N$, let $\Omega=\{\mathbf{z} \in \mathbb{R}^N \mid \sigma(\mathbf{a}^\top \mathbf{z})=\sigma(\mathbf{b}^\top \mathbf{z})\}$. If $\mathbf{a}\neq \mathbf{b}$, then $\Omega$ is of measure zero.
\end{lemma}
\begin{proof}   
Assume that $\Omega$ is not of measure zero. Since $\sigma(\mathbf{a}^\top \mathbf{z})-\sigma(\mathbf{b}^\top \mathbf{z})$ is an analytic function of $\mathbf{z}$, $\Omega$ must be $\mathbb{R}^N$ according to Lemma \ref{lem 2}. That is, $\sigma(\mathbf{a}^\top \mathbf{z})=\sigma(\mathbf{b}^\top \mathbf{z})$ for all $\mathbf{z} \in \mathbb{R}^N$. We will show that this leads to a contradiction.

If $\mathbf{a}=\mathbf{0}$ or $\mathbf{b}=\mathbf{0}$, assume $\mathbf{a}=\mathbf{0}$ without loss of generality. Then $\mathbf{b} \neq \mathbf{0}$, and there exists $\mathbf{z}_0$ such that $\mathbf{b}^\top \mathbf{z}_0=1$. Therefore, for any $\lambda\in\mathbb{R}$, we have $\sigma(\lambda)=\sigma(\mathbf{b}^\top (\lambda \mathbf{z}_0))=\sigma(\mathbf{a}^\top (\lambda \mathbf{z}_0))=\sigma(0)$. Thus, $\sigma$ is a constant function, leading to a contradiction.

If $\mathbf{a}, \mathbf{b}\neq \mathbf{0}$, there exists $\mathbf{z}_0$ such that $\mathbf{a}^\top\mathbf{z}_0 = 1$, $\mathbf{b}^\top \mathbf{z}_0 = q$ with $q\not=1$ and $q\not=0$. Without loss of generality, we assume $q<1$ (otherwise we swap $\mathbf{a}$ and $\mathbf{b}$). Then we have $\sigma(\lambda)=\sigma(\mathbf{a}^\top(\lambda\mathbf{z}_0)) = \sigma(\mathbf{b}^\top(\lambda\mathbf{z}_0)) = \sigma(q\lambda)$ for any $\lambda \in \mathbb{R}$. This implies
\begin{equation}
\label{eq::proof_exp_grow}
    \sigma'(\lambda) = \lim_{\delta\rightarrow 0} \frac{\sigma(\lambda + \delta)-\sigma(\lambda)}{\delta} = q\cdot\lim_{\delta\rightarrow 0} \frac{\sigma(q\lambda + q\delta )-\sigma(q\lambda)}{q\delta} = q\sigma'(q\lambda).
\end{equation}
Since $\sigma$ is not a constant function, there exists $\lambda \not= 0$ with $\sigma'(\lambda) \not=0$. By applying \eqref{eq::proof_exp_grow} successively, we have $\sigma'(\lambda) = q^k \sigma'(q^k\lambda)$ for any $k\in \mathbb{N}_+$. As $q<1$, we have $\sigma'(\lambda) = \lim_{k\rightarrow\infty}q^k\sigma'(q^k\lambda) = 0$, leading to a contradiction. We conclude that for any $\mathbf{a}\neq \mathbf{b}$, $\Omega$ cannot be $\mathbb{R}^N$ and therefore must be of measure zero.
\end{proof}
\begin{lemma}\label{lemma::data_equivalence}
Suppose that $\mathbf{x}^{(1)}, \mathbf{x}^{(2)}, \cdots, \mathbf{x}^{(N)} \in \mathbb{R}^{d_0}$ satisfy Assumption \ref{ass::data}. Then the two sets $\Omega =\{\mathbf{w} \in \mathbb{R}^{d_0} \mid \exists n' \not= n \text{ s.t. } \mathbf{w}^\top \mathbf{x}^{(n)} = \mathbf{w}^\top \mathbf{x}^{(n')}\}$ and $\Omega' = \{\mathbf{w} \in \mathbb{R}^{d_0} \mid \exists n \text{ s.t. } \mathbf{w}^\top \mathbf{x}^{(n)} = 0\}$ both have measure zero in $\mathbb{R}^{d_0}$. That is, for generic $\mathbf{w} \in \mathbb{R}^{d_0}$, $\mathbf{w}^\top \mathbf{x}^{(n)}$'s are distinct from each other, and are all non-zero.
\end{lemma}
\begin{proof}
Note that $\Omega$ is the union of $N(N-1)$ sets $\Omega_{n,n'} \triangleq \{ \mathbf{w}\in \mathbb{R}^{d_0} \mid   \mathbf{w}^\top \mathbf{x}^{(n)} = \mathbf{w}^\top \mathbf{x}^{(n')} \}$ with $n \not= n'$. Since $\mathbf{x}^{(1)},\mathbf{x}^{(2)},\cdots, \mathbf{x}^{(N)}$ are distinct from each other, each $\Omega_{n,n'}$ is a hyperplane, thus a zero-measure set. Therefore, $\Omega$ is also a zero-measure set.

Similarly, $\Omega'$ is the union of $N$ sets $\Omega'_{n} \triangleq \{ \mathbf{w}\in \mathbb{R}^{d_0} \mid \mathbf{w}^\top \mathbf{x}^{(n)} = 0 \}$. Since $\mathbf{x}^{(n)} \not= \mathbf{0}$, each $\Omega'_n$ is a hyperplane, thus a zero-measure set. Therefore, $\Omega'$ is of measure zero.
\end{proof}

Based on the above lemmas, below we show that for a deep neural networks satisfying Assumption 1 and Assumption 2, the output matrix of the last hidden layer has full column rank for almost all $(W_1, \cdots, W_H)$.
\begin{proposition}\label{extended prop 2}
Given a fully connected neural network with $H$ hidden layers, activation function $\sigma_h$ for each hidden layer, and empirical loss function $E(W)=l(Y, W_{H+1}T_H)$. Let $\Omega=\{(W_1, \cdots, W_{H})\mid \text{rank}(T_H) < N\}$. Suppose that Assumption \ref{ass::overall} holds and $\sigma_h$ satisfies Assumption \ref{ass::activation} for all $1\leq h \leq H$, then $\Omega$ is a zero-measure set.
\end{proposition}
\begin{proof}
We first handle the trivial case of $N=1$. In this case, $\Omega=\{(W_1, \cdots, W_H)\mid \text{rank}(T_H)=0\}$. Therefore, it is sufficient to show that $T_H\not=\mathbf{0}$ for almost all $(W_1, \cdots, W_H)$. Note that $T_H=\sigma_H(W_HT_{H-1})$, and since $\sigma_H\not\equiv0$ as it satisfies Assumption \ref{ass::activation}, the set of $W_H$ such that $T_H=\mathbf{0}$ is of measure zero. This implies that $\Omega$ is of measure zero. 

In the remaining proof, we always assume $N\geq2$. Note that when $N\geq2$, $\sigma_h$ cannot be a constant function since there exists $l\in\mathbb{N}_+$ such that $\sigma_h^{(l)}\not \equiv 0$. Thus, Lemma \ref{lem 3} applies.

In the following, we prove a slightly stronger result that the output matrix of the last hidden layer $T_H$ not only has full column rank but also, if regarded as an input data matrix, satisfies Assumption \ref{ass::data} for almost all $(W_1, \cdots, W_H)$. 

Specifically, denote $W_{1:h} = (W_1,W_2, \cdots, W_h)$, i.e., all the weights of the first $h$ hidden layers. Define
\begin{subequations}
\begin{gather}
\Omega_h = \{W_{1:h} \mid \mathrm{rank}(T_h)<\min\{d_h, N\} \}\\
\hat{\Omega}_{h} = \{W_{1:h} \mid \exists n \not= n', \quad \!\! s.t. \quad\!\! (T_h)_{n}=(T_h)_{n'}\}\cup \{W_{1:h} \mid \exists n, \quad \!\! s.t. \quad\!\! (T_h)_{n}  = \mathbf{0}\}
\end{gather}
\end{subequations}
where $(T_h)_n$ denotes the $n$-th column of $T_h$. $\Omega_h$ is the set of $W_{1:h}$ such that the output matrix of the $h$-th hidden layer is not full rank. $\hat{\Omega}_h$ is the set of $W_{1:h}$ such that $T_h$ has identical columns or has an all-zero column. That is, for any $W_{1:h} \in \hat{\Omega}_h$, the resulting $T_h$, if regarded as an input data matrix, violates Assumption \ref{ass::data}. 

In the following, we prove by induction that $\Phi_h \triangleq \Omega_h \cup \hat{\Omega}_h$ is of measure zero for all $1 \leq h \leq H $. Noting that $\Omega = \Omega_H\subseteq\Phi_H$, showing that $\Phi_H$ is of measure zero completes the proof.

We first prove that $\tilde{\Omega}_1$ is of measure zero. Let $(\mathbf{w}_1)_i^\top$ be the $i$-th row of $W_1$, $i = 1,2, \cdots, d_1$. Then, we have $(T_1)_{i,n}= \sigma((\mathbf{w}_1)_i^\top \mathbf{x}^{(n)})$. From Assumption \ref{ass::data}, $\mathbf{x}^{(1)}, \mathbf{x}^{(2)}, \cdots, \mathbf{x}^{(N)}$ are non-zero and distinct from each other. Note that $\sigma$ is not identically zero, by Lemma \ref{lem 2}, $(T_1)_{i,n} \not= 0$ for generic $W_1$. From Lemma \ref{lem 3}, for generic $W_1$, $(T_1)_{i,n} \not= (T_1)_{i,n'}$ for any $n \not = n'$. Therefore, $\hat{\Omega}_1$ is of measure zero.
  
Next we show that $\Omega_1$ is also of measure zero. From Lemma \ref{lemma::data_equivalence}, there exists an invertible matrix $Q\in \mathbb{R}^{d_0 \times d_0}$, such that $\tilde{X} \triangleq QX \in \mathbb{R}^{d_0 \times N}$ satifies
$\tilde{x}_{1,n}\not= 0$ and $\tilde{x}_{1,n}\not=\tilde{x}_{1,n'}$ for all $n,n' = 1,2,\cdots, N$ with $n \not= n'$. That is, for the first row of $\tilde{X}$, each entry is non-zero, and all entries are distinct from each other. We define $\tilde{W}_1 \triangleq W_1 Q^{-1}$, yielding $\tilde{W}_1\tilde{X} = W_1X$. Then, it suffices to show that
\begin{equation}
    \tilde{\Omega}_1 \triangleq \left\{\tilde{W}_1\in \mathbb{R}^{d_1\times d_0} \mid \text{rank}(T_1) < \min\{d_1,N\}\right\}
\end{equation}
is of measure zero.

Notice that $T_1\in\mathbb{R}^{d_1\times N}$. We first prove a special case with $d_1=N$ and then generalize the result to the case with $d_1<N$ or $d_1> N$.

\textbf{Case 1}: If $d_1=N$, we prove that $\tilde{\Omega}_1$ is of measure zero by induction on $N$. The conclusion is obvious when $N=1$.

For $N>1$, since $f(\tilde{W}_1)\triangleq\text{det}(T_1) = \text{det}(\sigma_1(\tilde{W}_1\tilde{X}))$ is analytic with respect to $\tilde{W}_1$, from Lemma \ref{lem 2} we know that $\tilde{\Omega}_1$ is either $\mathbb{R}^{N\times d_0}$ or a zero-measure set.

Suppose that $\tilde{\Omega}_1=\mathbb{R}^{N \times d_0}$, i.e., $f(\tilde{W}_1)\equiv0, ~\forall \tilde{W}_1\in\mathbb{R}^{N\times d_0}$. Note that the first row of $\tilde{X}$ has distinct entries. For any $l \in \mathbb{N}$, the $l$-th order partial derivative of $f(\tilde{W}_1)$ with respect to $(\tilde{W}_1)_{1,1}$ is given by\footnote{Notice that here we are taking derivatives with respect to a scalar $(\tilde{W}_1)_{1,1}$, i.e., the entry in the first row and the first column of $\tilde{W}_1$.}
{\small
\begin{align*}
&  \frac{\partial^lf}{\partial (\tilde{W}_1)_{1,1}^l} = \det(G_l(\tilde{W}_1))    \\
&  \triangleq \det \begin{pmatrix}
 (\tilde{x}^{(1)}_1)^l\sigma_1^{(l)}((\tilde{\mathbf{w}}_1)_1^\top \tilde{\mathbf{x}}^{(1)}) & (\tilde{\mathbf{x}}^{(1)}_2)^l\sigma_1^{(l)}((\tilde{\mathbf{w}}_1)_1^\top \tilde{\mathbf{x}}^{(2)}) & \cdots & (\tilde{\mathbf{x}}^{(1)}_N)^l\sigma_1^{(l)}((\tilde{\mathbf{w}}_1)_1^\top \tilde{\mathbf{x}}^{(N)}) \\
\sigma_1((\tilde{\mathbf{w}}_1)_2^\top \tilde{\mathbf{x}}^{(1)}) & \sigma_1((\tilde{\mathbf{w}}_1)_2^\top \tilde{\mathbf{x}}^{(2)}) & \cdots & \sigma_1((\tilde{\mathbf{w}}_1)_2^\top \tilde{\mathbf{x}}^{(N)}) \\
\vdots & \vdots & ~ & \vdots \\
\sigma_1((\tilde{\mathbf{w}}_1)_N^\top \tilde{\mathbf{x}}^{(1)}) & \sigma_1((\tilde{\mathbf{w}}_1)_N^\top \tilde{\mathbf{x}}^{(2)}) & \cdots & \sigma_1((\tilde{\mathbf{w}}_1)_N^\top \tilde{\mathbf{x}}^{(N)})
    \end{pmatrix}
\end{align*}
}
\!where $(\tilde{\mathbf{w}}_1)^\top_i$ denotes the $i$-th row of $\tilde{W}_1$, and $\tilde{\mathbf{x}}^{(n)}$ denotes the $n$-th column of $\tilde{X}$. As $f(\tilde{W}_1)\equiv0$, we have $\det(G_l(\tilde{W}_1))\equiv0$ for all $\tilde{W}_1\in\mathbb{R}^{N\times d_0}$ and $l \in \mathbb{N}$.

Denote the $n$-th row of $G_l(\tilde{W}_1)$ by 
\begin{equation}
    \mathbf{u}_n=[\sigma_1((\tilde{\mathbf{w}}_1)_n^\top \tilde{\mathbf{x}}^{(1)}), \cdots, \sigma_1((\tilde{\mathbf{w}}_1)_n^\top \tilde{\mathbf{x}}^{(N)})]^\top
\end{equation}
for all $n=2,\cdots, N$. We show that there exists some $(\tilde{\mathbf{w}}_1)_2^\top, \cdots, (\tilde{\mathbf{w}}_1)_N^\top$ such that $\mathbf{u}_2, \cdots, \mathbf{u}_N$ are linearly independent. Denote
\begin{equation}
\hat{\mathbf{u}}_n=\left[\sigma_1((\tilde{\mathbf{w}}_1)_n^\top \tilde{\mathbf{x}}^{(1)}), \cdots, \sigma_1((\tilde{\mathbf{w}}_1)_n^\top \tilde{\mathbf{x}}^{(N-1)})\right]^\top \in \mathbb{R}^{N-1}, n=2, \cdots, N
\end{equation}
and $\hat{G}=[\hat{\mathbf{u}}_2, \cdots, \hat{\mathbf{u}}_N]^\top \in \mathbb{R}^{(N-1) \times (N-1)}$. According to the induction hypothesis, the set $\{\left((\tilde{\mathbf{w}}_1)_2^\top, \cdots, (\tilde{\mathbf{w}}_1)_N^\top\right) \mid \text{det}(\hat{G})\neq0\}$ is zero-measure in $\mathbb{R}^{(N-1)\times d_0}$. Therefore, there exist $(\tilde{\mathbf{w}}_1)_2^\top, \cdots, (\tilde{\mathbf{w}}_1)_N^\top$ such that $\hat{\mathbf{u}}_2, \cdots, \hat{\mathbf{u}}_N$ are linearly independent, which also implies that $\mathbf{u}_2, \cdots, \mathbf{u}_N$ are linearly independent.

Now we have found $(\tilde{\mathbf{w}}_1)_2^\top, \cdots, (\tilde{\mathbf{w}}_1)_N^\top$ such that $\mathbf{u}_2, \cdots, \mathbf{u}_N$ are linearly independent. Fix $(\tilde{\mathbf{w}}_1)_2^\top, \cdots, (\tilde{\mathbf{w}}_1)_N^\top$ and let $(\tilde{\mathbf{w}}_1)_1^\top=\mathbf{0}$. Denote the first row of $G_l$ as $\mathbf{a}_l$. Since $\text{det}(G_l(\tilde{W})) = 0$ for any $l\geq 0$, $\mathbf{a}_l$ must be a linear combination of $\mathbf{u}_2, \cdots, \mathbf{u}_N$ for any $l \in \mathbb{N}$, so all $\mathbf{a}_l$'s lie in an $(N-1)$-dimension subspace of $\mathbb{R}^N$. However, according to Lemma \ref{lem 1}, the $N$ vectors $\mathbf{a}_{l_0}, \cdots, \mathbf{a}_{l_{N-1}}$ are linearly independent, where $l_0, l_1, \cdots, l_{N-1}$ form an arithmetic sequence specified in Assumption \ref{ass::activation}. This leads to a contradiction. 
Therefore, $\tilde{\Omega}_1$ cannot be $\mathbb{R}^N$, so it must be a zero-measure set.

\textbf{Case 2}: If $d_1<N$, we select the first $d_1$ columns of $T_1$ and obtain a sub-matrix $T'_1\in\mathbb{R}^{d_1\times d_1}$. Let $\tilde{\Omega}_1'=\{\tilde{W}_1\in \mathbb{R}^{d_1\times d_0}\mid \text{rank}(T'_1)<d_1\}$. We can show that $\tilde{\Omega}_1'$ is a zero-measure set by applying a similar analysis to $T'_1$ as in Case 1. Notice that for any $\tilde{W}_1\in\tilde{\Omega}_1$, any $d_1 \times d_1$ sub-matrix of $T_1$ should be singular. Therefore, $\tilde{\Omega}_1$ is a subset of $\tilde{\Omega}_1'$, and it should also have measure zero. 
    
\textbf{Case 3}: If $d_1> N$, we select the first $N$ rows of $T_1$ and obtain a sub-matrix $T'_1\in\mathbb{R}^{N \times N}$. Similarly, let $\tilde{W}'_1\in\mathbb{R}^{N\times d_0}$ be the first $N$ rows of $\tilde{W}_1$. Let $\tilde{\Omega}_1'=\{\tilde{W}'_1\in \mathbb{R}^{N\times d_0}\mid \text{rank}(T'_1)<N\}$. Following a similar analysis as in Case 1, $\tilde{\Omega}_1'$ is of measure zero in $\mathbb{R}^{N\times d_0}$. Note that for any $\tilde{W}_1\in\tilde{\Omega}_1$, its submatrix consisting of the first $N$ rows is in $\tilde{\Omega}_1'$. Thus, $\tilde{\Omega}_1$ is of measure zero in $\mathbb{R}^{d_1 \times d_0}$.

Combining the above analysis, $\tilde{\Omega}_1$ is of measure zero. This implies that $\Omega_1$ is also of measure zero. As a result, $\Phi_1 = \Omega_1 \cup \hat{\Omega}_1$ is of measure zero.
	
Now assume that $\Phi_{h-1}$ is of measure zero. Then $\Phi_h$ can be decomposed into
\begin{equation}
\label{eq 2}
	 \Phi_h =   \left\{W_{1:h}\in \Phi_h \mid W_{1:(h-1)}\in \Phi_{h-1}\right\} \cup
\left\{W_{1:h}\in \Phi_h \mid W_{1:(h-1)}\not \in \Phi_{h-1}\right\}
\end{equation}
By the induction hypothesis, the first component of the set union in (\ref{eq 2}) has measure zero in the space of $W_{1:h}$. Moreover, for $W_{1:(h-1)}\notin\hat{\Omega}_{h-1}$, the resulting $T_{h-1}$, if regarded as an input data matrix, satisfies Assumption \ref{ass::data}. Noting that $\sigma_h$ satisfies Assumption \ref{ass::activation}, following a similar procedure as in the case of $h=1$, we obtain that the set of $W_h$ satisfying $(W_{1:(h-1)}, W_h)\in\Omega_h \cup \hat{\Omega}_h$ has measure zero in $\mathbb{R}^{d_{h}\times d_{h-1}}$. This implies that the second component of the set union in (\ref{eq 2}) also has measure zero in the space of $W_{1:h}$. Therefore, $\Phi_h$ is of measure zero for all $1\leq h\leq H$.

By assumption, $d_H \geq N$, and hence $\Omega = \Omega_H\subseteq\Phi_H$. We complete the proof.
\end{proof}

We note that Proposition \ref{extended prop 2} states that the output matrix of the hidden layer is generically full-column-rank, i.e., an extended notion of Definition \ref{def::full_rank_reachable}. 

We next present another Lemma used in the proof of Theorem \ref{thm::deep-convex-analytic}.
\begin{lemma}\label{lem::convex_strictlydecreasing}
Let $f:\mathbb{R}^m\rightarrow \mathbb{R}$ be a convex function. From any $\mathbf{p}_0 \in \mathbb{R}^m$ that is not a global minimum of $f$, there exists a strictly decreasing path to the global infimum in the sense of Definition \ref{def::path_to_globalinf}.
\end{lemma}
\begin{proof}
Note that $f$ is a convex function defined on the whole space of $\mathbb{R}^m$, and hence is continuous on $\mathbb{R}^m$. Since $\mathbf{p}_0$ is not a global minimum of $f$, there exists a sequence of points $\{\mathbf{p}_i\}^{\infty}_{i=0}$ (starting from $\mathbf{p}_0$) such that $f(\mathbf{p}_{i+1}) < f(\mathbf{p}_i)$ for all $i\in\mathbb{N}$, and 
$\lim_{i\rightarrow\infty} f(\mathbf{p}_i)=\inf_{\mathbf{p}\in\mathbb{R}^m}f(\mathbf{p})$.

Next, we construct another sequence $\{\mathbf{p}'_i\}^{\infty}_{i=1}$ as follows. First, let $\mathbf{p}'_0 = \mathbf{p}_0$, then clearly we have $f(\mathbf{p}'_0) = f(\mathbf{p}_0)$. Suppose that $\mathbf{p}'_i$ is determined and $f(\mathbf{p}'_i) = f(\mathbf{p}_i)$. We consider the line segment between $\mathbf{p}'_i$ and $\mathbf{p}_{i+1}$, i.e., the mapping
\begin{equation}
\label{eq::convex_dp_line1}
    g_i(t) \triangleq (1-t) \mathbf{p}'_i + t \mathbf{p}_{i+1}, ~~ t\in [0,1].
\end{equation}
If $f(g_i(t))$ is strictly decreasing for $t\in[0,1]$, we directly set $\mathbf{p}'_{i+1} = \mathbf{p}_{i+1}$. Otherwise, we define $t^*$ as the smallest minimum of $g_i(t)$ in $[0,1]$, i.e.,
\begin{equation}
    t^* \triangleq\min\left\{t\in[0, 1]: f(g_i(t))=\min_{t'\in[0, 1]}f(g_i(t'))\right\}.
\end{equation} Since $f(g_i(0))=f(\mathbf{p}'_i) = f(\mathbf{p}_i) > f(\mathbf{p}_{i+1})=f(g_i(1))$, we have $t^*>0$.

We show that $f(g_i(t))$ is strictly decreasing for $t\in[0,t_*]$. If this is not the case, then there exist $0\leq t_1< t_2\leq t^*$ satisfying $f(g_i(t_1))\leq f(g_i(t_2))$. We have
\begin{subequations}
\label{eq::convex_dp_ieq}
    \begin{align}
    \label{eq::convex_dp_ieq0}
    f(g_i(t_2)) &= f\left((1-t_2)\mathbf{p}'_i+\mathbf{p}_{i+1}\right)\\
    \label{eq::convex_dp_ieq1}
    &=f\left(\frac{t^*-t_2}{t^*-t_1}g_i(t_1)+\frac{t_2-t_1}{t^*-t_1}g_i(t^*)\right) \\
    \label{eq::convex_dp_ieq2}
    &\leq \frac{t^*-t_2}{t^*-t_1}f(g_i(t_1))+\frac{t_2-t_1}{t^*-t_1}f(g_i(t^*))\\
    \label{eq::convex_dp_ieq3}
    & \leq \frac{t^*-t_2}{t^*-t_1}f(g_i(t_2))+\frac{t_2-t_1}{t^*-t_1}f(g_i(t_2)) = f(g_i(t_2))
    \end{align}
\end{subequations}
where \eqref{eq::convex_dp_ieq0}, \eqref{eq::convex_dp_ieq1} follow from \eqref{eq::convex_dp_line1}, and \eqref{eq::convex_dp_ieq2} is by the convexity of $f$. The equality of \eqref{eq::convex_dp_ieq3} holds only if $f(g_i(t_1))=f(g_i(t^*))= f(g_i(t_2))$, implying that $t_1$ is also a minimum of $f(g_i(t))$. However, $t_1  < t^*$, which is contradictory to the definition of $t^*$. Therefore, $g_i(t)$ must be strictly decreasing for $t\in [0,t^*]$. Note that $f$ and $g_i$ are both continuous mappings. From $f(g_i(0)) = f(\mathbf{p}_i) > f(g_i(1)) \geq f(g_i(t^*))$, there exists $t' \in (0,t^*]$ such that $f(g_i(t')) = f(\mathbf{p}_{i+1})$. Then, $f(g_i(t))$ is also strictly decreasing for $t\in[0,t']$. We set $\mathbf{p}'_{i+1} = g_i(t')$, and hence we have $f(\mathbf{p}'_{i+1}) = f(\mathbf{p}_{i+1})$.

Finally, we denote $t_i \triangleq 1-2^{-i}$ and construct a mapping $g:[0,1) \rightarrow \mathbb{R}^m$ as
\begin{equation}
    g(t) = \begin{cases}
    \mathbf{p}'_i,& t = t_i, ~~i\in \mathbb{N} \\
    \frac{t_{i+1}-t}{t_{i+1}-t_i}\mathbf{p}'_i +   \frac{t-t_i}{t_{i+1}-t_i}\mathbf{p}'_{i+1},& t \in \left(t_i, t_{i+1}\right), ~~i\in \mathbb{N}
    \end{cases}.
\end{equation}
Note that $g(t)$ connects $\{\mathbf{p}'_i\}^{\infty}_{i=1}$ sequentially by line segments. Thus $g$ is a continuous mapping. Also, we have $g(0) = \mathbf{p}'_0 = \mathbf{p}_0$ and $f(\lim_{t\rightarrow 1}g(t)) = \lim_{i\rightarrow\infty} f(\mathbf{p}'_i)=\lim_{i\rightarrow\infty} f(\mathbf{p}_i) = \inf_{\mathbf{p}\in\mathbb{R}^m}f(\mathbf{p})$. Further, we note that $f(\mathbf{p}'_i) > f(\mathbf{p}'_{i+1})$, and from the construction of $\{\mathbf{p}'\}^{\infty}_{i=1}$, $f(g(t))$ is strictly decreasing for $t\in [t_i,t_{i+1}]$, $\forall i\in\mathbb{N}$. This implies that $f(g(t))$ is strictly decreasing for $t\in[0,1)$. Therefore, $g$ is a strictly decreasing path to the global infimum. We complete the proof.
\end{proof}

Based on Proposition \ref{extended prop 2} and Lemma \ref{lem::convex_strictlydecreasing}, we now proceed to prove Theorem \ref{thm::deep-convex-analytic}.

\begin{proof}[Proof of Theorem \ref{thm::deep-convex-analytic}]
According to Proposition \ref{extended prop 2}, all $W$'s that entail a non-full-rank $T_H$ only constitute a zero-measure set. Therefore, for an arbitrary initial weight $W^o = (W^o_1, \cdots, W^o_{H+1})$ and an arbitrarily small $\delta>0$, there exists $W^p = (W^p_1, \cdots, W^p_{H+1}) \in B(W^o, \delta)$ such that the corresponding $T_H$ is full-rank. Since $d_H\geq N$, we have $\text{rank}(T^p_H)=N$.
	
In what follows, we show that starting from $W^p$, there exists a continuous path along which the empirical loss $E(W)$ strictly decreases to $\inf_{W} E(W)$. 
Denote $W_{1:H}=(W_1, \cdots, W_H)$, i.e., the weights in the first $H$ layers. By the feed-forward operation \eqref{eq 1}, $T_H$ is a function of $W_{1:H}$. Thus, $E(W)$ can be rewritten as $l(Y, W_{H+1}T_H(W_{1:H}))$. Since $l(Y, \hat{Y})$ is convex with respect to $\hat{Y}$, for any $W_{H+1}$,  $W'_{H+1}$, and $\lambda\in[0, 1]$, we have
\begin{align}
&l\left(Y, \left(\lambda W_{H+1}+(1-\lambda)W'_{H+1}\right)T_H(W_{1:H})\right)\nonumber\\
= &~ l\left(Y, \lambda W_{H+1}T_H(W_{1:H})+(1-\lambda)W'_{H+1}T_H(W_{1:H})\right) \nonumber \\
\leq& ~\lambda l\left(Y, W_{H+1}T_H(W_{1:H})\right)+(1-\lambda)l\left(Y, W'_{H+1}T_H(W_{1:H})\right).
\end{align}
Thus, with the weights to the first $H$ hidden layers fixed, $E(W)$ is convex with respect to $W_{H+1}$. By Lemma \ref{lem::convex_strictlydecreasing}, starting from $W^p$, we can find a strictly decreasing path towards $\inf_{W_{H+1}}l(Y, W_{H+1}T_H(W^p_{1:H}))$ by fixing $W_{1:H} = W^p_{1:H}$ and moving along $W_{H+1}$.
Moreover, since $T_H(W^p_{1:H}) \in \mathbb{R}^{d_{H} \times N}$ is full-column-rank, for any $\hat{Y}\in\mathbb{R}^ {d_{H+1}\times N}$, there exists $W_{H+1}$ such that $W_{H+1}T_H(W^p_{1:H})=\hat{Y}$, yielding
\begin{equation}
\inf_{W_{H+1}}l(Y, W_{H+1}T_H(W^p_{1:H}))=\inf_{\hat{Y}}l(Y, \hat{Y})=\inf_W E(W).
\end{equation}
Therefore, the constructed path is strictly decreasing towards $\inf_W E(W)$.
\end{proof}

Finally, we provide a formal proof of Proposition \ref{prop::PT_weaklyglobal} (presented in Section \ref{sec::main-results}), which states that Property PT implies being weakly global.
\begin{proof}[Proof of Proposition \ref{prop::PT_weaklyglobal}]
We prove the result by contraposition. Assume in contrast that there exists a sub-optimal set-wise strict local minimum of $f(W)$, denoted by $\mathcal{W}$. Note by Definition \ref{def_2}, $\mathcal{W}$ is a compact set. Let $\mathcal{W}_\delta=\{W'\mid \inf_{W\in\mathcal{W}}\|W'-W\|_2\leq\delta \}$, then there exists $\delta>0$ such that for all $W\in\mathcal{W}$ and $W'\in\mathcal{W}_\delta\setminus \mathcal{W}$, $f(W)<f(W')$. Denote $\partial \mathcal{W}_\delta$ as the boundary of $\mathcal{W}_\delta$. Note that both $\mathcal{W}_\delta$ and $\partial\mathcal{W}_\delta$ are compact, there exists $W^* \in \partial \mathcal{W}_\delta$ such that $f(W^*) = \inf_{W'\in\partial \mathcal{W}_\delta}f(W')$. Moreover, $f(W^*)=\sup_{W\in\mathcal{W}}f(W)+\varepsilon$ for some $\varepsilon>0$. 
	
Consider an arbitrary point $W^o\in\mathcal{W}$. Since $f(W)$ is a continuous function, there exists $\delta>\delta_0>0$ such that for any $W'\in B(W^o, \delta_0)$, $|f(W')-f(W^o)|<\varepsilon/2$. From the previous analysis, we can find $W^p \in B(W^o, \delta_0)$ such that there exists a strictly decreasing path from $W^p$ to the global infimum of $f(W)$. Since $\mathcal{W}$ is a sub-optimal local minimum, $\inf_{W\in \mathcal{W}_{\delta}} f(W) > \inf_{W} f(W)$. Therefore, the above strictly decreasing path starting from $W^p$ must pass through the boundary $\partial \mathcal{W}_\delta$. However, $f(W^p)<f(W^o)+\varepsilon/2< \sup_{W\in\mathcal{W}}f(W)+\varepsilon=f(W^*)=\inf_{W'\in\partial \mathcal{W}_\delta}f(W')$. This implies that the considered path can never be strictly decreasing, leading to a contradiction. Therefore, we conclude that there is no sub-optimal set-wise strict local minima, and therefore $f(W)$ is a weakly global function.
\end{proof}

\subsection{Proof of Theorem \ref{thm::continuous} (Step 2 and Step 3)}
We proceed to \textbf{Step 2}, i.e., to show that the activation function in Assumption \ref{ass::activation} can approximate any continuous function. This allows us to extend Corollary \ref{cor::deep-convex-analytic} to all continuous activation functions without dealing with them directly. Specifically, we use a mathematical trick to approximate an arbitrary continuous function by a class of analytical functions. 

\begin{lemma}\label{lemma::continuous_activation}
For any continuous function $f: \mathbb{R} \rightarrow \mathbb{R}$, there exists a sequence of functions $(f_k)_{k\in\mathbb{N}}$, all satisfying Assumption \ref{ass::activation}, such that $f_k$ converges to $f$ uniformly. 
\end{lemma}

\begin{proof}
The proof of Lemma \ref{lemma::continuous_activation} consists of two parts. In the first part, we show that the function class specified by Assumption \ref{ass::activation} is dense in the space of analytic functions. In the second part, following the fact that the space of analytic functions is a dense set in the space of continuous functions, we prove that the function class specified by Assumption \ref{ass::activation} is also dense in the space of continuous functions.

For the first part, we consider an arbitrary analytic function $g: \mathbb{R} \rightarrow \mathbb{R}$. Define $f_{k}(x) = g(x) + (\sin x + \cos x)/[s(k+1)]$.
Clearly, $f_k$ is analytic for any $k \in \mathbb{N}$ and $s \not= 0$. Further, the $n$-th order derivative of $f_k$ at zero is given by $f^{(n)}_k(0) = g^{(n)}(0) + (-1)^{\left\lfloor \frac{n}{2}\right \rfloor}/[s(k+1)].$
We next show that there exists $s \not= 0$ such that all $f_k$'s satisfy Assumption $\ref{ass::activation}$. Consider the following two cases.

\textbf{Case 1}: $g^{(n)}(0) = 0$ for all $0 \leq n \leq N-1$. For any $s \not= 0$, since $g^{(n)}(0) = 0$, we have $f^{(n)}_k(0)  = (-1)^{\left\lfloor \frac{n}{2}\right\rfloor}/[s(k+1)] \not= 0$ for all $n = 0,1,\cdots, N-1$. Since $\{0, 1, \cdots, N-1\}$ is an arithmetic sequence, all $f_k$'s satisfy Assumption $\ref{ass::activation}$. 

\textbf{Case 2}: $g^{(n)}(0) \not = 0$ for some $0 \leq n \leq N$. Since $g^{(n)}(0) \not= 0$ for at least one $n \in \{0,1,\cdots,N-1\}$, we can define 
\begin{equation}
\label{deltamin}
\delta_{\min} = \min\left\{|g^{(n)}(0)|: 0\leq n \leq N, \quad\!\!g^{(n)}(0) \not= 0\right\}
\end{equation}
i.e., the minimum non-zero absolute value of $g^{(n)}(0)$ for $0\leq n\leq N-1$. Clearly, $\delta_{\min}>0$. Letting $s = 2/\delta_{\min}$, we have $f^{(n)}_k(0) = g^{(n)}(0) + (-1)^{\left\lfloor \frac{n}{2}\right \rfloor}\cdot \delta_{\min}/[2(k+1)]$.
For $g^{(n)}(0) = 0$, we have $|f^{(n)}_k(0)|=\delta_{\min} / [2(k+1)] > 0$. For $g^{(n)}(0) \not= 0$, we have
\begin{equation*}
    \left|f^{(n)}_k(0)\right|\geq \left|g^{(n)}(0)\right| - \left|\frac{\delta_{\min}}{2(k+1)}(-1)^{\left\lfloor \frac{n}{2}\right \rfloor}\right|\geq \delta_{\min} - \frac{\delta_{\min}}{2(k+1)} = \frac{\delta_{\min}(2k+1)}{2(k+1)} > 0
\end{equation*}
by the definition of $\delta_{\min}$ in \eqref{deltamin}. Therefore, $f^{(n)}_k(0) \not=0$ for all $n=0,1,\cdots, N-1$ and $k\in \mathbb{N}$. Since $\{0,1,\cdots, N-1\}$ is an arithmetic sequence, all $f_k$'s satisfy Assumption \ref{ass::activation}.

We now prove the uniform convergence of $f_{k}$ to $f$ for any $s \not = 0$. Specifically, for any $\epsilon > 0$, we have 
\begin{equation}
\left|f_k(x) - g(x)\right| = \frac{1}{s(k+1)}\left|\sin x + \cos x\right|\leq \frac{\sqrt{2}}{s(k+1)}  < \epsilon
\end{equation}
for all $k > \sqrt{2}/(\epsilon s) - 1$ and $x \in \mathbb{R}$. Therefore, $f_k$ converges uniformly to $g$.

We conclude that the function class specified by Assumption \ref{ass::activation} is dense in the space of analytic functions.

Now we come to the second part. By the Carleman Approximation Theorem \cite{kaplan1955approximation}, the space of analytic functions is dense in the space of continuous functions. That is, for any continuous function $f: \mathbb{R} \rightarrow \mathbb{R}$, there exists a sequence of analytic functions $(g_k)_{k\in\mathbb{N}}$ converging to $f$ uniformly. Following the idea of Cantor's diagonal argument, we can construct a sequence of functions, all satisfying Assumption \ref{ass::activation}, that also converges to $f$.

Specifically, by the analysis in the first part, for each analytic function $g_k$ we can construct $(f^{(k)}_j)_{j\in \mathbb{N}}$, all satisfying Assumption \ref{ass::activation}, such that $f^{(k)}_j$ converges to $g_k$ uniformly. Further, we can require that for each $k\in \mathbb{N}$,
\begin{equation}
\label{eq::seq_construct}
\left|f^{(k)}_j(x) - g_k(x)\right| \leq \frac{1}{k+1}, \quad \!\! \forall x\in \mathbb{R}, \quad \!\! j \in \mathbb{N}.
\end{equation}
In fact, if \eqref{eq::seq_construct} is not satisfied, we can always delete a finite number of functions at the beginning of $(f_j^{(k)})_{j\in\mathbb{N}}$, so as to produce a new sequence that meets \eqref{eq::seq_construct}. Now consider the sequence $(f^{(k)}_k)_{k\in \mathbb{N}}$. Since $g_k$ converges to $f$ uniformly, for any $\epsilon > 0$, there exists a $K_1 \in \mathbb{N}$ such that $|g_k(x) - f(x)|\leq \epsilon/2$ for any $k\geq K_1$ and $x \in \mathbb{R}$. Then, for any $k > \max\{K_1, 2/\epsilon -1\}$, we have
\begin{equation}
\left|f^{(k)}_k(x) - f(x)\right| \leq    \left |f^{(k)}_k(x) - g_k(x)\right| + \left|g_k(x) - f(x)\right| 
\leq \frac{1}{k+1} +\epsilon/2 \leq \epsilon.
\end{equation}
Therefore, $f^{(k)}_k$ converges to $f$ uniformly. We complete the proof.
\end{proof}

Lemma \ref{lemma::continuous_activation} means that the analytic functions satisfying Assumption \ref{ass::activation} constitute a dense set (in the sense of uniform convergence) in the space of continuous functions. Note that a deep neural network contains compositions of activation functions. Subsequently, we introduce the notion of compact convergence and show that compact convergence is preserved under the composition of functions.
\begin{definition}[Compact convergence]
A sequence of functions $(f_k)_{k\in \mathbb{N}}: S \rightarrow \mathbb{R}^m$ is said to compactly converge to $f: S \rightarrow \mathbb{R}^m$ if, for every compact subset $K \subset S$, $f_k \rightarrow f$ uniformly on $K$.
\end{definition}
\begin{lemma}\label{lemma::nesting}
Consider two continuous functions $f: S\rightarrow \mathbb{R}^n$ and $g: \mathbb{R}^{n} \rightarrow \mathbb{R}$, where $S \subset \mathbb{R}^m$ is a compact set. Suppose that there exist two sequences of functions $(f_k)_{k\in\mathbb{N}}$ and $(g_k)_{k\in \mathbb{N}}$, such that $f_k$ uniformly converges to $f$ on $S$, and $g_k$ compactly converges to $g$ on $\mathbb{R}^{n}$. Then, $g_k\circ f_k$ converges to $g\circ f$ uniformly on $S$.
\end{lemma}
\begin{proof}
Let $D \subset \mathbb{R}^n$ be the range of $f$ on $S$. Since $S$ is compact and $f$ is continuous, $D$ is compact. Define $D' = \left\{\mathbf{z} \in \mathbb{R}^n\mid \exists \mathbf{z}_0 \in D, \quad\!\! ||\mathbf{z}-\mathbf{z}_0||_2  \leq 1\right\}$. Note that $D'$ is also compact. Then, since $g$ is continuous, its restriction on $D'$ is uniformly continuous. That is, for any $\epsilon >0$, there exists $\delta > 0$ such that
\begin{equation}
\label{eq::lemma_nesting0}
|g(\mathbf{z}_1) - g(\mathbf{z}_2)| \leq \epsilon/2, \quad \! \forall \mathbf{z}_1, \mathbf{z}_2 \in D', \quad \!\! ||\mathbf{z}_1 - \mathbf{z}_2||_2 \leq \delta.
\end{equation} 
Further, since $f_k$ converges to $f$ uniformly on $S$, there exists $K_1 \in \mathbb{N}$ such that
\begin{equation}
\label{eq::lemma_nesting1}
||f_k(\mathbf{x}) - f(\mathbf{x})||_2 \leq \min \left\{1, \delta \right\}, \quad\! \forall k \geq K_1, \quad\!\! \mathbf{x}\in S.
\end{equation}
By the definition of $D'$, \eqref{eq::lemma_nesting1} implies $f_k(\mathbf{x}) \in D'$ for all $k \geq K_1$ and $\mathbf{x} \in S$. Also, as $g_k$ compactly converges to $g$, $g_k$ uniformly converges to $g$ on $D'$. Then, there exists $K_2 \in \mathbb{N}$ such that $|g_k(\mathbf{z})-g(\mathbf{z})|\leq \epsilon/2$ for all $k \geq K_2$ and $\mathbf{z} \in D'$. Then, for any $k \geq \max\{K_1,K_2\}$ and $\mathbf{x} \in S$, we have
\begin{subequations}
\begin{align}
\left|g_k\left(f_k(\mathbf{x})\right) - g\left(f(\mathbf{x})\right)\right| &\leq 
\left|g_k\left(f_k(\mathbf{x})\right) - g\left(f_k(\mathbf{x})\right)\right|+\left|g\left(f_k(\mathbf{x})\right) - g\left(f(\mathbf{x})\right)\right| \\
\label{eq::lemma_nesting2_1}
&\leq \epsilon/2 + \epsilon/2 \leq \epsilon.
\end{align}
\end{subequations}
We conclude that $g_k\circ f_k$ converges to $g \circ f$ uniformly on $S$.
\end{proof}

Now we come to \textbf{Step 3}, i.e., to show that the property of being a weakly global function is preserved under compact convergence. Lemma \ref{lemma::continuous_activation} and Lemma \ref{lemma::nesting} allow us to approximate a neural network with any continuous activation function by a sequence of neural networks under Assumption \ref{ass::activation}. We show that the weakly global property is preserved under such approximation. 

We first present a modification of a result in \cite{josz2018theory}.
\begin{proposition}\label{prop::compact_convergence}
Consider a sequence of functions $(f_k)_{k\in\mathbb{N}}$ and a function $f$, all from $S \subset \mathbb{R}^m$ to $\mathbb{R}$. If, $f_k \rightarrow f$ compactly and if every $f_k$ is weakly global on $S$, then $f$ is also a weakly global function on $S$.
\end{proposition}
\textbf{Remark}: Proposition \ref{prop::compact_convergence} is slightly different from its original version in \cite{josz2018theory}: here we assume that $f_k$'s are weakly global functions instead of global functions. Nevertheless, we can still prove that $f$ is weakly global by using similar techniques as in \cite{josz2018theory}. We also note that Property PT is NOT preserved under such approximation. A sequence of functions with Property PT may converge compactly to a function with a spurious valley that violates Property PT. For instance\footnote{This example is provided by an anonymous reviewer of the submitted version. We sincerely thank this reviewer for the contribution.}, define $f_k: \mathbb{R}^2\rightarrow\mathbb{R}$ as: $f_k(x_1, x_2)\triangleq |x_1|+\max\{0, 1-2|x_1-1|\}\max\{0, \min\{1, k-|x_2|\}\}$ for $k\in\mathbb{N}$. The sequence $\{f_k\}$ converges compactly to $f(x_1, x_2)\triangleq |x_1|+\max\{0, 1-2|x_1-1|\}$ which contains an unbounded spurious valley $\{(x_1, x_2)|x_1=1.5\}$. Starting from any point in this valley, there are no strictly decreasing paths to global infimum after any sufficiently small perturbation. Thus, Property PT does not holds for $f$.
\begin{proof}[Proof of Proposition \ref{prop::compact_convergence}]
Note that $f_k$ converges compactly towards $f$. Since $S \subset \mathbb{R}^m$ and $\mathbb{R}^m$ is a compactly generated space, it follows that $f$ is continuous. We prove that $f$ is a weakly global function by contradiction.

Suppose that $S_M \subset S$ is a sub-optimal set-wise strict local minimum of $f$. There exists $\epsilon > 0$ such that the uniform neighborhood $V \triangleq \{ \mathbf{y} \in S | \exists\, \mathbf{x} \in S_M: \|\mathbf{x}-\mathbf{y}\|_2 \leq \epsilon \}$ satisfies $f(\mathbf{x}) < f(\mathbf{y})$ for all $\mathbf{x} \in S_M$ and $\mathbf{y} \in V \setminus S_M$. Since $f$ is continuous on the compact set $S_M$, it attains a minimal value on it, say $\inf_{S_M} f \triangleq \alpha + \inf_{S} f$ where $\alpha > 0$ since $S_M$ is not a global minimum. Consider a compact set $V \subset K \subset S$ such that $\inf_{K} f \leq \alpha/2 + \inf_{S} f$. Denote the boundary of $V$ by $\partial V$. Since $\partial V$ is a compact set, $f$ attains a minimal value on it, say $\inf_{\partial V} f \triangleq \beta + \inf_{S_M} f$ where $\beta > 0$ by strict optimality. Let $\gamma \triangleq \min \{ \alpha/2 , \beta \}$. For a sufficiently large value of $k$, compact convergence implies that $|f_k(\mathbf{y})-f(\mathbf{y})| \leq \gamma /3$ for all $\mathbf{y} \in K$. Since the function $f_k$ is continuous on $V$, it attains a minimum, say $\mathbf{z}' \in V$. Consider the compact set defined by $Z \triangleq \{ \mathbf{z} \in V ~|~ f(\mathbf{z}) = f(\mathbf{z}') \}$. Therefore, for any $\mathbf{z}\in Z$,
\begin{subequations}
	\begin{align}
	f_k(\mathbf{z}) &\leq \gamma/3 + \inf_{V} f \leq \beta/3 + \inf_{V} f  < 2\beta/3 + \inf_{V} f  \\
	& \leq - \gamma/3 + \beta + \inf_V f \leq  - \gamma/3 + \inf_{\partial V}f \leq \inf_{\partial V} f_k.
	\end{align}
\end{subequations}
Thus, $\mathbf{z} \in \text{int}(V)$, where $\text{int}(\cdot)$ denotes the interior of a set. So $Z\subseteq \text{int}(V)$. Since both $Z$ and $\partial V$ are compact, we have $d(\partial V, Z)>0$. We now proceed to show by contradiction that $Z$ is a strict local minimum of $f_k$. Assume that for all $\epsilon' > 0$, there exists $\mathbf{y}' \in S \setminus Z$ satisfying $d(\mathbf{y}', Z) \leq \epsilon'$ such that $f_k(\mathbf{z}) \geq f_k(\mathbf{y}')$ for some $\mathbf{z}\in Z$. We can choose $\epsilon'<d(\partial V, Z)$ to guarantee that $\mathbf{y}'$ belongs to $V$ since $Z \subseteq \text{int}(V)$. The point $\mathbf{y}'$ then contradicts the minimality of $Z$ on $V$. This means that $Z \subset V$ is a strict local minimum of $f_k$. Now, observe that for any $\mathbf{z}\in Z$,
\begin{subequations}
\begin{align}
	\inf_K f_k &\leq \gamma/3 + \inf_K f \leq \gamma/3 + \alpha/2 + \inf_S f \leq 2\alpha/3 + \inf_S f < 5\alpha/6 + \inf_S f \\
	&\leq \alpha - \gamma/3 + \inf_S f  = - \gamma/3 + \inf_{S_M} f  = -\gamma/3 + \inf_V f \leq \inf_V f_k \leq f_k(\mathbf{z}).
\end{align} 
\end{subequations}
Thus, $Z$ is not a global minimum of $f_k$. This contradicts the fact that $f_k$ is a weakly global function. Therefore, $f$ must also be a weakly global function.
\end{proof}

Based on Proposition \ref{prop::compact_convergence}, to prove Theorem \ref{thm::continuous} it suffices to find a sequence of weakly global functions that compactly converges to the loss function $E(W)$. Below we present the formal proof of Theorem \ref{thm::continuous}.
\begin{proof}[Proof of Theorem \ref{thm::continuous}]
We denote the considered network by $\mathcal{N}$. From Lemma \ref{lemma::continuous_activation}, for any $1\leq h \leq H$, there exist a sequence of activation functions $(\sigma_{h,k})_{k\in\mathbb{N}}$, each satisfying Assumption \ref{ass::activation} and uniformly converging to $\sigma_h$. For each $k \in \mathbb{N}$, we construct a neural network, referred to as $\mathcal{N}_k$, by replacing the activation function of the $h$-th hidden layer with $\sigma_{h,k}$, $h = 1, 2, \cdots, H$. We assume the training dataset of $\mathcal{N}_k$ to be identical to that of $\mathcal{N}$. We also denote the output matrix of the $h$-th hidden layer by $T^{(k)}_h$ and the empirical loss by $E_k(W) = l(Y, W_{H+1}T^{(k)}_H).$ From Theorem \ref{thm::deep-convex-analytic}, $E_k$ is a weakly global function with respect to $W$, $\forall k\in \mathbb{N}$. In what follows, we prove that the sequence of the empirical loss functions $(E_k)_{k\in \mathbb{N}}$ compactly converges to $E$. 

Consider an arbitrary compact subset $S$ in the space of $W$. For any $W\in S$, define $t^{(k)}_{h,i,n}(W) = (T^{(k)}_h)_{i,n}$ and $t_{h,i,n}(W) = (T_h)_{i,n}$ for any $k \in \mathbb{N}$, $1\leq h \leq H$, $1 \leq i \leq d_h$, and $1 \leq n \leq N$. That is, we rewrite the output of each neuron in the hidden layers as a function of $W$. We prove by induction that every sequence $(t^{(k)}_{h,i,n})_{k\in \mathbb{N}}$ converges to $t_{h,i,n}$ uniformly on $S$.

For $h =1$, we have
\begin{equation*}
t^{(k)}_{1,i,n}(W) = \sigma_{1,k}\left(\sum^{d_0}_{j=1}(W_1)_{i,j}x_{j,n}\right), \quad \!\!
t_{1,i,n}(W) = \sigma_1\left(\sum^{d_0}_{j=1}(W_1)_{i,j}x_{j,n}\right).
\end{equation*}
Since $\sigma_{1,k}$ uniformly converges to $\sigma_1$, $t^{(k)}_{1,j,n}$ also uniformly converges to $t_{1,j,n}$ on $S$ for all $1\leq j\leq d_1$, $1 \leq n \leq N$.

For $h>1$, assume that $t^{(k)}_{h-1,i,n}$ uniformly converges to $t_{h-1,i,n}$ on $S$ for all $1\leq i\leq d_{h-1}$, $1 \leq n \leq N$. For the $h$-th layer, we have
\begin{gather*}
t^{(k)}_{h,i,n}(W) =  \sigma_{h,k}\left(\sum^{d_{h-1}}_{j=1}(W_h)_{i,j}\left(T^{(k)}_{h-1}\right)_{j,n}\right)=\sigma_{h,k}\left(\sum^{d_{h-1}}_{j=1}(W_h)_{i,j} t^{(k)}_{h-1,j,n}(W) \right) \\
t_{h,i,n}(W) = \sigma_h\left(\sum^{d_{h-1}}_{j=1}(W_h)_{i,j}(T_{h-1})_{j,n}\right) =\sigma_h\left(\sum^{d_{h-1}}_{j=1}(W_h)_{i,j} t_{h-1,j,n}(W) \right).
\end{gather*}
By the induction hypothesis, it is easy to show that $\sum^{d_{h-1}}_{j=1}(W_h)_{i,j} t^{(k)}_{h-1,j,n}(W)$ uniformly converges to $\sum^{d_{h-1}}_{j=1}(W_h)_{i,j} t_{h-1,j,n}(W)$ on $S$. Note that $(\sigma_{h,k})_{k\in\mathbb{N}}$ converges to $\sigma_h$ uniformly, and uniform convergence implies compact convergence. It directly follows from Lemma \ref{lemma::nesting} that $t^{(k)}_{h,i,n}(W)$ converges to $t_{h,i,n}(W)$ uniformly on $S$.

Therefore, we conclude that $t^{(k)}_{h,i,n}$ converges to $t_{h,i,n}$ uniformly on $S$ for every $1\leq h \leq H$, $1 \leq i \leq d_h$, and $1 \leq n \leq N$. Then we consider the empirical loss
\begin{equation}
E_k(W) = l\left(Y, W_{H+1}T^{(k)}_H\right), ~~
E(W) = l\left(Y, W_{H+1}T_H\right).
\end{equation}
As every component of $T^{(k)}_H$ converges uniformly to the corresponding component of $T_H$ on $S$, it can be shown that $W_{H+1}T^{(k)}_H$ converges uniformly to $W_{H+1}T_H$ on $S$. By Lemma $\ref{lemma::nesting}$, where we set both $g_k$ and $g$ to the loss function $l(Y,\cdot)$, we have that $E_k$ uniformly converges to $E$ on $S$. Noting that $S$ is an arbitrary compact subset in the space of $W$, the empirical loss $E_k$ converges to $E$ compactly on the space of $W$. Since $E_k(W)$ is a weakly global function for every $k\in \mathbb{N}$, by Proposition \ref{prop::compact_convergence}, $E(W)$ is also a weakly global function. We complete the proof.
\end{proof}

\section{Formal Proof of Negative Results and Computational Construction
of Bad Basins}\label{sec::proof_negative}
\subsection{Proof of Negative Results}
First, we characterize a group of sufficient conditions for the existence of sub-optimal strict local minima in narrow neural networks.
\begin{proposition}
\label{prop::suff_cond_bad_strict}
Consider a $1$-hidden-layer neural network with $m<N$ and twice continuously differentiable activation function $\sigma$. For input data $\mathbf{x} \in \mathbb{R}^N$, define
\begin{subequations}
\label{eq::definition_z}
\begin{align}
    \mathbf{z}_i & \triangleq \sigma(w_i\mathbf{x})\\
    \mathbf{z}'_i & \triangleq \mathbf{x} \circ \sigma'(w_i\mathbf{x})\\
    \mathbf{z}''_i & \triangleq \mathbf{x}\circ\mathbf{x}\circ\sigma''(w_i\mathbf{x})
\end{align}
\end{subequations}
$i=1,2,\cdots, m$. Suppose that there exists $\mathbf{w}\in\mathbb{R}^m$ satisfying the following conditions:
\begin{enumerate}[label={(\arabic*)}]
    \item The vectors $\{\mathbf{z}_j\}^{m}_{j=1}$ are linearly independent,
    \item $\mathbf{z}''_i$ is linearly independent of $\{\mathbf{z}_j\}^{m}_{j=1}\cup \{\mathbf{z}_j'\}^m_{j=1}$, $i = 1,2,\cdots, m$, and
    \item there exists $w_0\in\mathbb{R}$ such that $\mathbf{z}_0 \triangleq \sigma(w_0 \mathbf{x})$ is linearly independent of $\{\mathbf{z}_j\}^m_{j=1}\cup \{\mathbf{z}'_j\}^m_{j=1}$.
\end{enumerate}
Then, there exists a corresponding output data $\mathbf{y}\in\mathbb{R}^N$ such that the empirical loss has a sub-optimal strict local minimum.
\end{proposition}
\begin{proof}
By conditions (2) and (3), there exists a vector $\mathbf{a}\in \mathbb{R}^N$ that is orthogonal to all vectors in $\{\mathbf{z}_j\}^{m}_{j=1}\cup \{\mathbf{z}'_j\}^{m}_{j=1}$, and satisfies $\mathbf{a}^\top \mathbf{z}''_i \not = 0$, $i=1,2,\cdots, m$, $\mathbf{a}^\top \mathbf{z}_0 \not= 0$. We set each entry of $\mathbf{v}$ as $v_i = \mathbf{a}^\top \mathbf{z}''_i/|\mathbf{a}^\top \mathbf{z}''_i|$, $i=1,2,\cdots, m$. Then, we set
\begin{equation}
\label{eq::bad_strict_y_construct}
    \mathbf{y} = \hat{\mathbf{y}} - M\mathbf{a} =  \sum^m_{i=1}v_i \mathbf{z}_i -M\mathbf{a}
\end{equation}
where $M$ is a positive value to be determined later. In the remaining proof, we show that for appropriately chosen $M$, $(\mathbf{v}, \mathbf{w})$ is a sub-optimal strict local minimum of the empirical loss given data $(\mathbf{x}, \mathbf{y})$.

We first show that $(\mathbf{v},\mathbf{w})$ is a strict local minimum for any $M>0$. To see this, we consider a perturbed point $(\mathbf{v}', \mathbf{w}')$ that is sufficiently close to $(\mathbf{v}, \mathbf{w})$. Let $\hat{\mathbf{y}}'$ denote the network output of $(\mathbf{v}', \mathbf{w}')$. By Taylor expansion, We have
\begin{align}
\label{eq::bad_strict_y_difference}
\hat{\mathbf{y}}' - \hat{\mathbf{y}} &= \sum^m_{i=1}(v_i'-v_i)\sigma(w_i \mathbf{x}) + \sum^m_{i=1} v_i'\left[\sigma(w_i' \mathbf{x}) - \sigma(w_i \mathbf{x})\right] \nonumber \\
&=\sum^m_{i=1}  \Delta v_i \mathbf{z}_i + \sum^m_{i=1}(v_i + \Delta v_i)\left[ \Delta w_i \mathbf{z}_i' + \frac{1}{2} \Delta w_i^2 \mathbf{z}_i'' + \mathbf{o}(\Delta w_i^2)\right]
\end{align}
where 
$\Delta v_i \triangleq v_i' - v_i$, $\Delta w_i \triangleq w_i' - w_i$, and $\mathbf{o}(\cdot)$ denotes an infinitesimal vector with
$\lim_{t \rightarrow 0}\|\mathbf{o}(t)\|_2/|t| = 0$.
Next, we decompose the difference of empirical loss as
\begin{align}
\label{eq::bad_strict_loss_decompose}
    E(\mathbf{v}', \mathbf{w}') - E(\mathbf{v} , \mathbf{w}) & = \|\hat{\mathbf{y}}' - \mathbf{y}\|_2^2 - \|\hat{\mathbf{y}} - \mathbf{y}\|_2^2 \nonumber \\
    & = \|\hat{\mathbf{y}}'-\hat{\mathbf{y}}\|^2_2  + 2 (\hat{\mathbf{y}} -\mathbf{y})^\top(\hat{\mathbf{y}}' - \mathbf{\hat{y}}).
\end{align}
Consider the following two cases.

\textbf{Case 1}: $\mathbf{w}' \not= \mathbf{w}$. We have
\begin{subequations}
\begin{align}
\label{eq::bad_strict_case1_1}
(\hat{\mathbf{y}}-\mathbf{y})^\top(\hat{\mathbf{y}}'-\hat{\mathbf{y}}) &= M \mathbf{a}^\top(\hat{\mathbf{y}}'-\hat{\mathbf{y}}) \\
\label{eq::bad_strict_case1_2}
& =  M \sum^m_{i=1}(v_i + \Delta v_i)\left[\frac{1}{2}\Delta w_i^2  \cdot \mathbf{a}^\top \mathbf{z}_i'' + \mathbf{a}^\top \mathbf{o}(\Delta w_i^2)\right]
\end{align}
\end{subequations}
where \eqref{eq::bad_strict_case1_1} follows from \eqref{eq::bad_strict_y_construct}, while \eqref{eq::bad_strict_case1_2} follows from \eqref{eq::bad_strict_y_difference} and the orthogonality between $\mathbf{a}$ and $\{\mathbf{z}_i\}^{m}_{j=1}\cup \{\mathbf{z}'_j\}^{m}_{j=1}$. Note that by our choice of $\mathbf{v}$, $v_i$ is non-zero and is of the same sign with $\mathbf{a}^\top \mathbf{z}_i''$. Thus, for $|\Delta v_i|\leq |v_i|/2$, $v_i + \Delta v_i$ is of the same sign with $v_i$, yielding $(v_i+\Delta v_i)\cdot \mathbf{a}^\top \mathbf{z}_i'' > v_i\cdot  \mathbf{a}^\top \mathbf{z}_i''/2>0$. Also, we have $(v_i+\Delta v_i)\mathbf{a}^\top \mathbf{o}(\Delta w_i^2) \geq -|3v_i\mathbf{a}^\top \mathbf{o}(\Delta w_i^2)/2|$. Noting that there exists at least one $\Delta w_i \not= 0$, we have
\begin{equation}
(\hat{\mathbf{y}}-\mathbf{y})^\top(\hat{\mathbf{y}}'-\hat{\mathbf{y}}) > M \sum_{\Delta w_i \not= 0} \left[\frac{1}{4}\Delta w_i^2 \cdot v_i \cdot \mathbf{a}^\top \mathbf{z}_i'' -\left| \frac{3}{2}v_i\mathbf{a}^\top \mathbf{o}(\Delta w_i^2)\right|\right] > 0
\end{equation}
if all $|\Delta w_i|$'s are sufficiently small. By \eqref{eq::bad_strict_loss_decompose}, we have 
\begin{equation}
E(\mathbf{v}', \mathbf{w}') - E(\mathbf{v} , \mathbf{w}) \geq 2 (\hat{\mathbf{y}} -\mathbf{y})^\top(\hat{\mathbf{y}}' - \mathbf{\hat{y}}) >0.
\end{equation}

\textbf{Case 2}: $\mathbf{w}' = \mathbf{w}$, $\mathbf{v}' \not= \mathbf{v}$. In this case, we have $\Delta w_i = 0$ for all $1\leq i\leq m$, yielding
\begin{equation}
(\hat{\mathbf{y}}-\mathbf{y})^\top(\hat{\mathbf{y}}'-\hat{\mathbf{y}}) =  M \sum^m_{i=1}(v_i + \Delta v_i)\left[\frac{1}{2}\Delta w_i^2  \cdot \mathbf{a}^\top \mathbf{z}_i'' + \mathbf{a}^\top \mathbf{o}(\Delta w_i^2)\right] = 0.
\end{equation}
Thus, we have
\begin{equation}
E(\mathbf{v}', \mathbf{w}') - E(\mathbf{v} , \mathbf{w}) = \|\hat{\mathbf{y}}' - \hat{\mathbf{y}}\|_2^2 = \left\|\sum^m_{i=1} \Delta v_i \mathbf{z}_i\right\|^2_2 >0
\end{equation}
since $\Delta v_i$'s are not all-zero and $\{\mathbf{z}_i\}^m_{i=1}$ are linearly independent.

Combining both cases, we see that $(\mathbf{v},\mathbf{w})$ is a strict local minimum.

Now we consider another point $(\tilde{\mathbf{v}}, \tilde{\mathbf{w}})$ where $\tilde{w}_1 = w_0$ specified in condition (3), and $\tilde{v}_i =0$ for all $i=2,3,\cdots, m$. Denoting the corresponding network output by $\tilde{\mathbf{y}}$, we have
\begin{align}
E(\tilde{\mathbf{v}}, \tilde{\mathbf{w}}) &= \|\tilde{\mathbf{y}} - \mathbf{y}\|_2^2 = \| \tilde{v}_1 \mathbf{z}_0 - \mathbf{\hat{y}}+M\mathbf{a}\|_2^2\nonumber \\
&=  \tilde{v}_1^2\|\mathbf{z}_0\|^2_2 + \|\hat{\mathbf{y}}\|_2^2 +  M^2\|\mathbf{a}\|_2^2 - 2\tilde{v}_1 \cdot \hat{\mathbf{y}}^\top \mathbf{z}_0 + 2\tilde{v}_1 M  \mathbf{a}^\top\mathbf{z}_0 - 2 M\mathbf{a}^\top\hat{\mathbf{y}}.
\end{align}
We write $\tilde{v}_1 = r\cdot M$ where the value of $r$ will be determined later. By our choice of $\mathbf{a}$, we have $\mathbf{a}^\top \mathbf{z}_0\not= 0$ and $\mathbf{a}^\top \hat{\mathbf{y}} = \sum^m_{i=1}v_i \cdot \mathbf{a}^\top \mathbf{z}_i = 0$. Then,
\begin{equation}
\label{eq::bad_strict_suboptimal}
    E(\tilde{\mathbf{v}}, \tilde{\mathbf{w}}) - E(\mathbf{v}, \mathbf{w}) = r(\|\mathbf{z}_0\|^2_2 \cdot r+2\mathbf{a}^\top \mathbf{z}_0) \cdot M^2 -2r\hat{\mathbf{y}}^\top \mathbf{z}_0\cdot M+ \|\hat{\mathbf{y}}\|_2^2
\end{equation}
which is a quadratic function of $M$. Note that $\mathbf{a}^\top \mathbf{z}_0 \not=0$. If $\mathbf{a}^\top \mathbf{z}_0 > 0$, we set $r\in \left(-2(\mathbf{a}^\top \mathbf{z_0})/\|\mathbf{z}_0\|^2_2,0\right)$; otherwise, we set $r\in \left(0,-2(\mathbf{a}^\top \mathbf{z_0})/\|\mathbf{z}_0\|^2_2\right)$. Then 
\begin{equation}
    r(\|\mathbf{z}_0\|^2_2 \cdot r+2\mathbf{a}^\top \mathbf{z}_0) < 0.
\end{equation}
That is, the coefficient of $M^2$ in \eqref{eq::bad_strict_suboptimal} is negative. For sufficiently large $M>0$, we have $E(\tilde{\mathbf{v}}, \tilde{\mathbf{w}}) - E(\mathbf{v}, \mathbf{w}) < 0$, implying that the strict local minimum $(\mathbf{v},\mathbf{w})$ is indeed sub-optimal.
\end{proof}

Note that a sub-optimal strict local minimum always entails a bad basin. Based on Proposition \ref{prop::suff_cond_bad_strict}, we can prove Theorem \ref{thm::bad_basin_exists} and \ref{thm::bad_basin_smooth} by constructing neural networks satisfying conditions (1)-(3). The proof of Proposition \ref{prop::suff_cond_bad_strict} also suggests a concrete approach to construct a sub-optimal strict local minimum. In the next subsection, we show how to utilize this approach to computationally construct a sub-optimal basin.

Now, we present the formal proofs of negative results.

\begin{proof}[Proof of Theorem \ref{thm::bad_basin_exists}]
Since $\mathbf{x}$ has no zero entries and no identical entries, we can always find a weight vector $\mathbf{w}\in \mathbb{R}^m$ (without zero entries and identical entries)  such that $w_i x^{(j)} \not= w_{i'} x^{(j')}$ for any $i\not= i'$ or $j \not= j'$.

Subsequently, we design an appropriate activation function $\sigma$ such that the conditions in Proposition \ref{prop::suff_cond_bad_strict} hold. Since $w_i x^{(j)}$'s are different from each other, we can arbitrarily design (each entry of) $\mathbf{z}_i$, $\mathbf{z}'_i$, $\mathbf{z}''_i$. Specifically, for any $\mathbf{p}^{(i)}, \mathbf{q}^{(i)}, \mathbf{r}^{(i)}\in \mathbb{R}^N$, $i=1,2,\cdots, m$, there exists a twice continuously differentiable function $\sigma(\cdot)$ (e.g., a high-order polynomial) such that
\begin{equation}
\sigma(w_i x_j) = \mathbf{p}^{(i)}_j, ~~ \sigma'(w_i x_j) = \frac{\mathbf{q}^{(i)}_j}{x_j}, ~~ \sigma''(w_i x_j) = \frac{\mathbf{r}^{(i)}_j}{x_j^2}, ~~ j=1,2,\cdots, N,
\end{equation}
implying $\mathbf{z}_i = \mathbf{p}^{(i)}$, $\mathbf{z}'_i = \mathbf{q}^{(i)}$, $\mathbf{z}''_i = \mathbf{r}^{(i)}$. Similarly, we can arbitrarily design $\mathbf{z}_0=\sigma(w_0\mathbf{x})$ as long as $w_0$ is different from any entry of $\mathbf{w}$. We first let $\{\mathbf{z}_j\}^{m}_{j=1}\cup \{\mathbf{z}'_j\}^{m}_{j=1}$ span a linear subspace with dimension less than $N$, and then design $\mathbf{z}''_i$, $w_0$ to satisfy conditions (1)-(3) in Proposition \ref{prop::suff_cond_bad_strict}.

Finally, by Proposition \ref{prop::suff_cond_bad_strict} there exists $\mathbf{y} \in \mathbb{R}^N$ such that the empirical loss has a sub-optimal strict local minimum, and hence has a bad basin.
\end{proof}

\begin{proof}[Proof of Theorem \ref{thm::homogeneous_no_bad_basin}]
Assume on the contrary that the empirical loss has a set-wise strict local minimum $S$. By Definition \ref{def_1} $S$ is a compact set. Consider an arbitrary $(\mathbf{v}, \mathbf{w})\in S$. We distinguish between the following two cases.

\textbf{Case 1}: $\mathbf{v}=\mathbf{0}$. For any $\Delta \mathbf{w}$ and  $\mathbf{w}' = \mathbf{w} + \Delta \mathbf{w}$, we have 
\begin{equation}
\label{eq::homogenuous_case_1}
    \hat{\mathbf{y}}'^\top = \mathbf{v}^\top\sigma(\mathbf{w}'\mathbf{x}^\top) = \mathbf{0}^\top = \mathbf{v}^\top\sigma(\mathbf{w}\mathbf{x}^\top) = \hat{\mathbf{y}}^\top.
\end{equation}
yielding $E(\mathbf{v}, \mathbf{w}) = E(\mathbf{v}, \mathbf{w}')$. The line segment between $(\mathbf{v},\mathbf{w})$ and $(\mathbf{v},\mathbf{w}')$ is a continuous path with constant empirical loss. Thus $(\mathbf{v}, \mathbf{w}') \in S$. Note that \eqref{eq::homogenuous_case_1} holds for arbitrary $\Delta \mathbf{w}$, the norm of $\mathbf{w}'$ can be made arbitrarily large, and hence $S$ cannot be a compact set, a contradiction!

\textbf{Case 2}: $\mathbf{\mathbf{v}}\not=\mathbf{0}$. For any $M>0$, consider $(\mathbf{v}', \mathbf{w}') = (M\cdot \mathbf{v}, \mathbf{w}/M)$. By assumption, we have
\begin{equation}
    \hat{\mathbf{y}}'^\top =  M\cdot \mathbf{v}^\top\sigma\left(\frac{1}{M}\cdot\mathbf{w}'\mathbf{x}^\top\right) = \mathbf{v}^\top\sigma(\mathbf{w}\mathbf{x}^\top) = \hat{\mathbf{y}}^\top.
\end{equation}
yielding $E(\mathbf{v}, \mathbf{w}) = E(\mathbf{v}, \mathbf{w}')$. By continuously changing $M$, we can obtain a continuous path between $(\mathbf{v}, \mathbf{w})$ and $(\mathbf{v}', \mathbf{w}')$. Thus $(\mathbf{v}', \mathbf{w}') \in S$. Note that $\mathbf{v}'$ can be arbitrarily scaled, $\|\mathbf{v}'\|_2$ can be made arbitrarily large, a contradiction to the compactness of $S$.

Therefore, the empirical loss has no set-wise strict local minimum.

Note that the proof can be easily extended to an arbitrarily deep network. In fact, if the empirical loss of a deep network with positive homogeneous activation has a set-wise strict local minimum, we can manipulate the weights of the last hidden layer to show a contradiction, following the same idea as above.
\end{proof}

\begin{proof}[Proof of Theorem \ref{thm::bad_basin_smooth}]
Note that the considered network has only 1 hidden neuron. For simplicity, we denote the network weights by $v,w\in\mathbb{R}$. Similarly to \eqref{eq::definition_z}, we define $\mathbf{z} \triangleq \sigma(w\mathbf{x})$, $\mathbf{z}' \triangleq \mathbf{x}\circ\sigma'(w\mathbf{x})$, and $\mathbf{z}'' \triangleq \mathbf{x}\circ \mathbf{x}\circ \sigma''(w\mathbf{x})$.

We set $w=0$. Since $\mathbf{z} \not= \mathbf{0}$, condition (1) of Proposition \ref{prop::suff_cond_bad_strict} holds. Next, consider the following matrix
\begin{equation}
    Z\triangleq \begin{pmatrix}
    \mathbf{z} \\
    \mathbf{z}' \\
    \mathbf{z}''
    \end{pmatrix} = 
    \begin{pmatrix}
    \sigma(0) & \sigma(0) & \cdots & \sigma(0) \\
    x^{(1)} \sigma'(0) & x^{(2)} \sigma'(0) & \cdots & x^{(N)}\sigma'(0) \\
    (x^{(1)})^2\sigma''(0) & (x^{(2)})^2\sigma''(0) & \cdots & (x^{(N)})^2\sigma''(0)
    \end{pmatrix} \in \mathbb{R}^{3\times N}.
\end{equation}
Note that $Z$ is a sub-matrix of an $N\times N$ Vandermonde matrix (with each row multiplied by a non-zero scalar), and therefore has full row rank. Thus, condition (2) of Proposition \ref{prop::suff_cond_bad_strict} holds.

Consider $\mathbf{z}_0 =\sigma(w_0\mathbf{x})$ with $w_0\not=0$. Note that $\sigma'(0) \not=0$, $\sigma(t)$ is not constant around $t=0$. Since $\mathbf{x}$ has no identical entries, $\mathbf{z_0}$ is not aligned with $\mathbf{z}=\sigma(0)\cdot \mathbf{1}$ for $w_0$ sufficiently close to zero. Further, note that $\lim_{w_0 \rightarrow 0} \mathbf{z}_0 = \mathbf{z}$. Since $\mathbf{z}$ is linearly independent of $\mathbf{z}'$, $\mathbf{z}_0$ is also linearly independent of $\mathbf{z}'$ for $w_0$ sufficiently close to zero. Therefore, condition (3) of Proposition \ref{prop::suff_cond_bad_strict} holds. 

By Proposition \ref{prop::suff_cond_bad_strict}, there exists $\mathbf{y}\in \mathbb{R}^N$ such that the considered network has a sub-optimal basin.
\end{proof}

\begin{proof}[Proof of Theorem \ref{thm::bad_basin_piecewise}]By the assumption $\sigma(0) = \sigma'(0)=0$, the activation function is ``flat'' around the origin with zero function value. Note that $\sigma$ is piecewise linear. Let $t_0\not=0$ be a nearest breakpoint (at which the slope of the function value changes) to the origin. Then, there exist $a \not=0$ and $\delta>0$ such that
\begin{equation}
    \sigma(t) = \begin{cases}
    0, & t\in[0,t_0]\\
    a(t - t_0), & t\in(t_0,t_0 + \delta]
    \end{cases}
\end{equation}
if $t_0>0$, and
\begin{equation}
    \sigma(t) = \begin{cases}
    a(t_0 - t), & t\in[t_0-\delta,t_0)\\
    0, & t\in[t_0,0]
    \end{cases}
\end{equation}
if $t_0<0$. Without loss of generality, we assume $a,t_0>0$. Otherwise, we can reverse the signs of the constructed weights correspondingly and apply the same analysis; see the last paragraph of this proof.

\textbf{Step 1 (Construction)}: Consider input data $\mathbf{x} \in \mathbb{R}^N$ satisfying $x_1>x_2>\cdots>x_N$ with $x_1/x_N < (t_0+\delta)/t_0$. Note that all such $\mathbf{x}$'s constitute a positive-measure set, as we can slightly perturb $\mathbf{x}$ without violating the above condition.

For any $0 \leq i \leq N$, we say a vector $\mathbf{a}\in\mathbb{R}^N$ has pattern-$i$ if its first $i$ entries are all positive and the remaining entries are all zero. Next, define $\mathcal{P}_i$ as the set of all $N$-dimensional vectors that have pattern-$i$. We further define
\begin{equation}
    u_i = \frac{t_0}{x_{i}}, ~~ 1\leq i\leq N, ~~u_{N+1} = \frac{t_0+\delta}{x_1}
\end{equation}
Then, from our choice of $\mathbf{x}$, we have $u_1<u_2<\cdots<u_N < u_{N+1}$. Further, for any $1\leq n \leq N$ we have
\begin{subequations}
\label{eq::property_ui}
\begin{equation}
    u_ix_n = \frac{t_0x_n}{x_i}\in \begin{cases}
    (t_0,t_0+\delta) & n<i \\
    (0,t_0) & n\geq i
    \end{cases}, ~~ 1\leq i \leq N
\end{equation}
an
\begin{equation}
    u_{N+1}x_n = \frac{(t_0+\delta)x_n}{x_1}\in (t_0,t_0+\delta].
\end{equation}
\end{subequations}

Consider a weight configuration $(\mathbf{v}, \mathbf{w})$ with
\begin{equation}
w_i = u_{2i+1}, ~~ v_i = 1, ~~ i=1,2,\cdots, m
\end{equation}
From \eqref{eq::property_ui}, we have that $\mathbf{z}_i\triangleq \sigma(w_i\mathbf{x})$ has pattern-($2i$) for any $i=1,2,\cdots, m$. Then, we set the output data $\mathbf{y}$ as
\begin{equation}
    \mathbf{y} = \hat{\mathbf{y}} - M\mathbf{e}_{2m+1} =  \sum^m_{i=1}v_i\mathbf{z}_i - M\mathbf{e}_{2m+1}
\end{equation}
where $M$ is a positive value to be determined later, and $\mathbf{e}_{2m+1}\in\mathbb{R}^N$ is the $(2m+1)$-th standard basis vector of $\mathbb{R}^N$ (noting that $2m+1\leq N$ by assumption). Then the epirical loss is given by
\begin{equation}
    E(\mathbf{v}, \mathbf{w}) = \|\hat{\mathbf{y}} - \mathbf{y}\|^2_2 = M^2.
\end{equation}

In what follows, we will show that $(\mathbf{v}, \mathbf{w})$ is a sub-optimal strict local minimum of the empirical loss.

\textbf{Step 2 (Local-Min)}: Let $(\mathbf{v}', \mathbf{w}')$ be another weight configuration in the neighbourhood of the constructed $(\mathbf{v}, \mathbf{w})$. 
Denote the output vectors of the hidden neurons and the output neuron by $\mathbf{z}'_i \triangleq \sigma(w'_i\mathbf{x})$ and $(\hat{\mathbf{y}}')^\top = (\mathbf{v}')^\top \sigma(\mathbf{w}'\mathbf{x}^\top)$, respectively. We assume that $\mathbf{w}'$ is sufficiently close to $\mathbf{w}$ such that $u_{2i} < w'_i < u_{2i+2}$ for all $i=1,2,\cdots, m$. From \eqref{eq::property_ui} we have $\mathbf{z}'_i \in \mathcal{P}_{2i}\cup \mathcal{P}_{2i+1}$. We also assume that $v_i'>0$ for all $i=1,2,\cdots, m$.

We discuss the following 6 cases.

\textbf{(i)} If $w'_m>w_m$, then $\mathbf{z}'_m$ has pattern-$(2m+1)$. Note that $\mathbf{z}'_i \in \mathcal{P}_{2i}\cup \mathcal{P}_{2i+1}$ for all $i<m$. Therefore, $\hat{\mathbf{y}}'\in \mathcal{P}_{2m+1}$, and $\hat{y}'_{2m+1} = v'_{m}\sigma(w'_m x_{2m+1}) > 0$. Thus,
\begin{equation}
    E(\mathbf{v}',\mathbf{w}') \geq \|\hat{y}'_{2m+1}-y_{2m+1}\|^2_2 = \|\hat{y}'_{2m+1} + M\|^2_2 > M^2 = E(\mathbf{v},\mathbf{w}).
\end{equation}

\textbf{(ii)} If $w'_m<w_m$, then $\mathbf{z}'_m$ has pattern-$2m$. Since $\mathbf{z}'_i\in \mathcal{P}_{2i}\cup \mathcal{P}_{2i+1}$ for all $i<m$ and $w'_m>u_{2m}$, we have
\begin{equation}
\label{eq::piece_basin_proof_case2}
    \frac{\hat{y}'_{2m}}{\hat{y}'_{2m-1}} = \frac{v'_m\sigma(w'_m x_{2m})}{\sum^m_{i=1} v'_i\sigma(w'_i x_{2m-1})} \leq \frac{\sigma(w'_m x_{2m})}{\sigma(w'_m x_{2m-1})}= \frac{w'_m x_{2m} - t_0}{w'_m x_{2m-1}- t_0}.
\end{equation}
Since $x_{2m}<x_{2m-1}$, the right hand side of \eqref{eq::piece_basin_proof_case2} is strictly increasing for $w'_m \in (u_{2m},w_m)$. We have
\begin{equation}
    \frac{\hat{y}'_{2m}}{\hat{y}'_{2m-1}} <\frac{w_m x_{2m} - t_0}{w_m x_{2m-1}- t_0} = \frac{\hat{y}_{2m}}{\hat{y}_{2m-1}},
\end{equation}
implying $\hat{\mathbf{y}}' \not= \hat{\mathbf{y}}$. Noting that $\hat{\mathbf{y}}', 
\hat{\mathbf{y}}\in \mathcal{P}_{2m}$, we have
\begin{equation}
    E(\mathbf{v}',\mathbf{w}') = \sum^{2m}_{i=1}\|\hat{y}'_i- \hat{y}_i\|_2^2 + M^2 > M^2 =  E(\mathbf{v},\mathbf{w}).
\end{equation}

\textbf{(iii)} If $w'_m=w_m$ and $v'_m\not=v_m$, then $\hat{\mathbf{y}}\in \mathcal{P}_{2m}$. Further,
\begin{equation}
    \hat{y}'_{2m} = v'_m \sigma(w_m x_{2m}) \not= v_m \sigma(w_m x_{2m}) = \hat{y}_{2m}.
\end{equation}
We have
\begin{equation}
    E(\mathbf{v}',\mathbf{w}') \geq \|\hat{y}'_{2m}- \hat{y}_{2m}\|_2^2 + M^2 > M^2 =  E(\mathbf{v},\mathbf{w}).
\end{equation}

\textbf{(iv)} If there exists $1\leq i<m$ such that $(v'_j,w'_j) = (v_j,w_j)$ for all $j> i$, and $w'_i >w_i$, then $\mathbf{z}'_i \in \mathcal{P}_{2i+1}$. We have
\begin{equation}
    \hat{y}'_{2i+1} = v'_i \sigma(w'_i x_{2i+1}) + \sum^m_{j=i+1} v_j \sigma(w_j x_{2i+1}) = v'_i \sigma(w'_i x_{2i+1}) + \hat{y}_{2i+1} > \hat{y}_{2i+1}. 
\end{equation}
Noting that $\hat{\mathbf{y}}', 
\hat{\mathbf{y}}\in \mathcal{P}_{2m}$, we have
\begin{equation}
    E(\mathbf{v}',\mathbf{w}') \geq \|\hat{y}'_{2i+1}- \hat{y}_{2i+1}\|_2^2 + M^2 > M^2 =  E(\mathbf{v},\mathbf{w}).
\end{equation}

\textbf{(v)} If there exists $1\leq i<m$ such that $(v'_j,w'_j) = (v_j,w_j)$ for all $j> i$, and $w'_i < w_i$, then $\mathbf{z}'_i \in \mathcal{P}_{2i}$. We have
\begin{align}
    \frac{\hat{y}'_{2i}- \sum^m_{j=i+1}v_j\sigma(w_j x_{2i})}{\hat{y}'_{2i-1}- \sum^m_{j=i+1}v_j\sigma(w_j x_{2i-1})} &= \frac{v'_i \sigma(w_i' x_{2i})}{v'_i \sigma(w_i' x_{2i-1}) + \sum^{i-1}_{k=1}v'_k \sigma(w_k' x_{2i-1})} \nonumber \\
    &\leq \frac{\sigma(w_i' x_{2i})}{ \sigma(w_i' x_{2i-1})} = \frac{w_i' x_{2i} - t_0}{w_i' x_{2i-1}-t_0}\nonumber \\
    &< \frac{w_i x_{2i} - t_0}{w_i x_{2i-1}-t_0} \nonumber \\
    &=\frac{\hat{y}_{2i}- \sum^m_{j=i+1}v_j\sigma(w_j x_{2i})}{\hat{y}_{2i-1}- \sum^m_{j=i+1}v_j\sigma(w_j x_{2i-1})},
\end{align}
implying $(\hat{y}'_{2i-1}, \hat{y}'_{2i})\not=(\hat{y}_{2i-1}, \hat{y}_{2i})$. Noting that $\hat{\mathbf{y}}', 
\hat{\mathbf{y}}\in \mathcal{P}_{2m}$, we have
\begin{equation}
    E(\mathbf{v}',\mathbf{w}') \geq \|\hat{y}'_{2i-1}- \hat{y}_{2i-1}\|_2^2 + \|\hat{y}'_{2i}- \hat{y}_{2i}\|_2^2 + M^2 > M^2 =  E(\mathbf{v},\mathbf{w}).
\end{equation}

\textbf{(vi)} If there exists $1\leq i<m$ such that $(v'_j,w'_j) = (v_j,w_j)$ for all $j> i$, and $w'_i = w_i$, $v'_i\not= v_i$, then $\mathbf{z}'_j \in \mathcal{P}_{2j}$ for all $j\geq i$. In this case we have
\begin{equation}
    \hat{y}'_{2i} =v'_i\sigma(w_i x_{2i}) + \sum^m_{j=i+1}v_j\sigma(w_j x_{2i}) \not= v_i\sigma(w_i x_{2i}) + \sum^m_{j=i+1}v_j\sigma(w_j x_{2i}) =  \hat{y}_{2i}.
\end{equation}
Therefore, 
\begin{equation}
    E(\mathbf{v}',\mathbf{w}') \geq \|\hat{y}'_{2i}- \hat{y}_{2i}\|_2^2 + M^2 > M^2 =  E(\mathbf{v},\mathbf{w}).
\end{equation}

Combining (i)-(vi), we see that $(\mathbf{v},\mathbf{w})$ is a strict local minimum of $E$.

\textbf{Step 3 (Sub-Optimality)}: Now consider a weight configuration $(\mathbf{v}^*, \mathbf{w}^*)$ with $w^*_m \in (u_{2m+1}, u_{2m+2})$, $v^*_m = -\alpha M $, and $w_i^*=v^*_i =0$ for all $1\leq i<m$. Here $\alpha$ is a positive value to be determined later. Then, the corresponding network output satisfies
\begin{equation}
    -\hat{\mathbf{y}}^* = \alpha M \sigma(w^*_m \mathbf{x}) \in \mathcal{P}_{2m+1}.
\end{equation}
We have
\begin{align}
\label{eq::piece_basin_proof_subopt}
    &~E(\mathbf{v}^*,\mathbf{w}^*)-E(\mathbf{v},\mathbf{w}) \nonumber \\
    =&~\|\hat{\mathbf{y}} - M\mathbf{e}_{2m+1} +\alpha M \sigma(w^*_m\mathbf{x}) \|_2^2 -M^2 \nonumber \\
     =&~ \sum^{2m}_{i=1}\left|\hat{y}_i+\alpha M \sigma(w^*_m x_i)\right|_2^2 + M^2\left|\alpha \sigma(w^*_m x_{2m+1})-1\right|^2_2 -M^2
\end{align}
which can be seen as a quadratic function of $M$ with the quadratic coefficient given by
\begin{equation}
    c_2 = \left[\sum^{2m}_{i=1}\sigma^2(w^*_m x_i) + \sigma^2(w^*_m x_{2m+1})\right]\alpha^2 - 2\sigma(w^*_m x_{2m+1})\alpha.
\end{equation}
We let
\begin{equation}
    0< \alpha < \frac{2\sigma(w^*_m x_{2m+1})}{\sum^{2m}_{i=1}\sigma^2(w^*_m x_i) + \sigma^2(w^*_m x_{2m+1})},
\end{equation}
then the quadratic coefficient of $M$ in \eqref{eq::piece_basin_proof_subopt} is negative. Thus, for sufficiently large $M$, \begin{equation}
    E(\mathbf{v}^*,\mathbf{w}^*)-E(\mathbf{v},\mathbf{w})<0.
\end{equation}
As a result, $(\mathbf{v},\mathbf{w})$ is a sub-optimal strict local minimum. Since a sub-optimal strict local minimum always entails a sub-optimal basin, we complete the proof.

Finally, we note that if $t_0<0$ and $a>0$, we can set $\tilde{\mathbf{w}} = -\mathbf{w}$ and show that $(\mathbf{v}, \tilde{\mathbf{w}})$ is a sub-optimal strict local minimum by the same analysis. If $t_0>0$ and $a<0$, we can set $\tilde{\mathbf{v}} = -\mathbf{v}$ and show that $(\tilde{\mathbf{v}}, \mathbf{w})$ is a sub-optimal strict local minimum. Moreover, if $t_0,a_0<0$, we can show that $(\tilde{\mathbf{v}}, \tilde{\mathbf{w}})$ is a sub-optimal strict local minimum.
\end{proof}

\subsection{Computational Construction of Sub-Optimal Basins}
Theorem \ref{thm::bad_basin_exists} is an existential statement in which we only care about the existence of a certain network that possesses a bad basin. One might be interested in finding a bad basin for a specific neural network and specific input data. In fact, our proof of Proposition \ref{prop::suff_cond_bad_strict} suggests a general computational approach to construct bad basins. In the following, we provide a concrete example for Swish activation.

\begin{example}[Bad basin for Swish network]\label{ex::swish}
Consider a $1$-hidden-layer single neuron network with Swish activation $\sigma(t)=t/(1+e^{-t})$. The network output can thus be represented as $\hat{\mathbf{y}}^\top=v \sigma(w\cdot \mathbf{x}^\top)$, $v,w \in \mathbb{R}$. We consider three data samples $\mathbf{x} = (x^{(1)}, x^{(2)}, x^{(3)})^\top=(0.1, 2, 3)^\top$ and choose $w=v=1$. It can be verified that conditions (1)-(3) in Proposition \ref{prop::suff_cond_bad_strict} are all satisfied. Then, according to the construction approach in the proof of Proposition \ref{prop::suff_cond_bad_strict}, we aim to find a vector $\mathbf{a}=(a_1, a_2, a_3)^\top$ such that 
\begin{subequations}
\label{eq::swish_example}
    \begin{align}
        \mathbf{a}^\top \mathbf{z} = &\sum_{n=1}^3a_n\cdot\frac{wx^{(n)}}{1+e^{-wx^{(n)}}}=0\\
        \mathbf{a}^\top \mathbf{z}'= &\sum_{n=1}^3a_nx^{(n)}\cdot\frac{1+(1+wx^{(n)})e^{-wx^{(n)}}}{(1+e^{-wx^{(n)}})^2}=0\\
        \mathbf{a}^\top \mathbf{z}''= &\sum^3_{n=1}a_n\cdot\frac{(2-wx^{(n)}+2e^{-wx^{(n)}}+wx^{(n)}e^{-wx^{(n)}})e^{-wx^{(n)}}}{(1+e^{-wx^{(n)}})^3}>0.
    \end{align}
\end{subequations}
We computationally find
a feasible solution $\mathbf{a}\approx(1, 0.0294, -0.0365)^\top$ to \eqref{eq::swish_example}. Then, we set $\mathbf{y}=v\sigma(w\cdot \mathbf{x})-100\mathbf{a} \approx(-99.9475, -1.1766, 6.5060)^\top$. In this way, we find a neural network with a sub-optimal strict local minimum $(v, w)=(1, 1)$. The loss surface of the network is visualized in Figure \ref{fig:swishex}.
\end{example}

\begin{figure}
    \centering
    \subfigure{\label{fig:swishex:subfig:a}
    \includegraphics[width=0.45\linewidth]{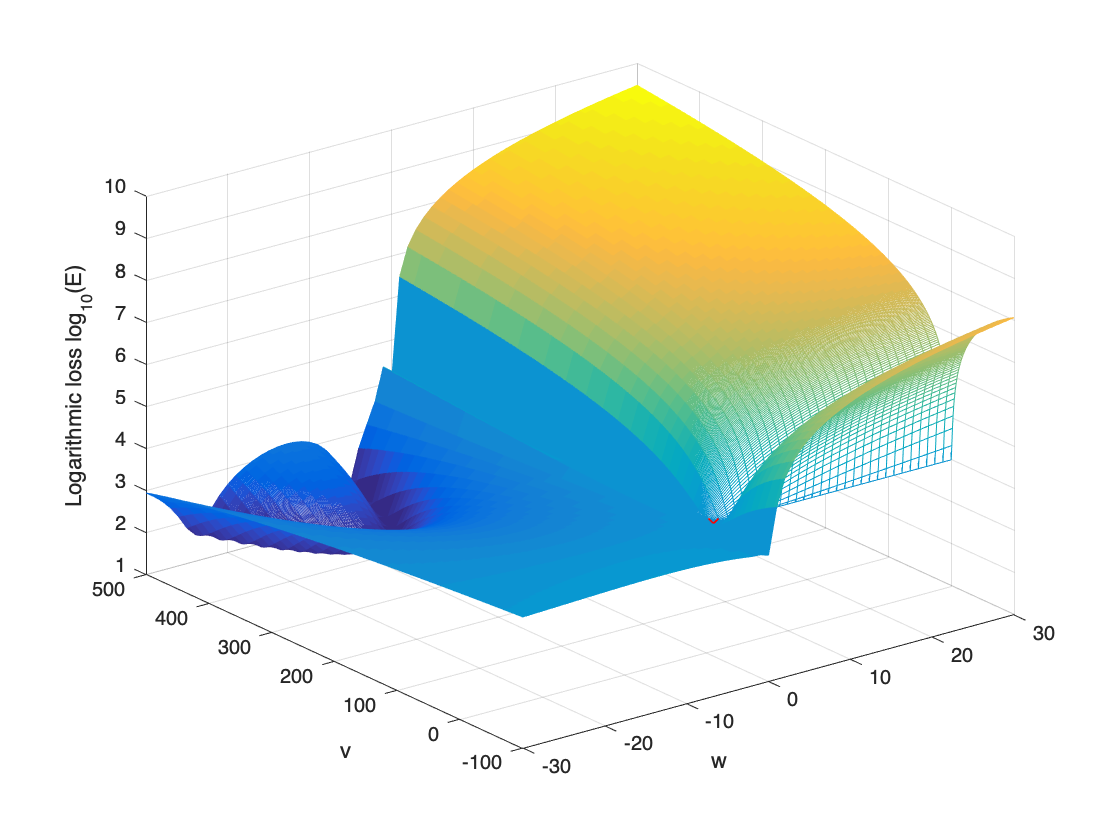}}
    \subfigure{\label{fig:swishex:subfig:b}
    \includegraphics[width=0.45\linewidth]{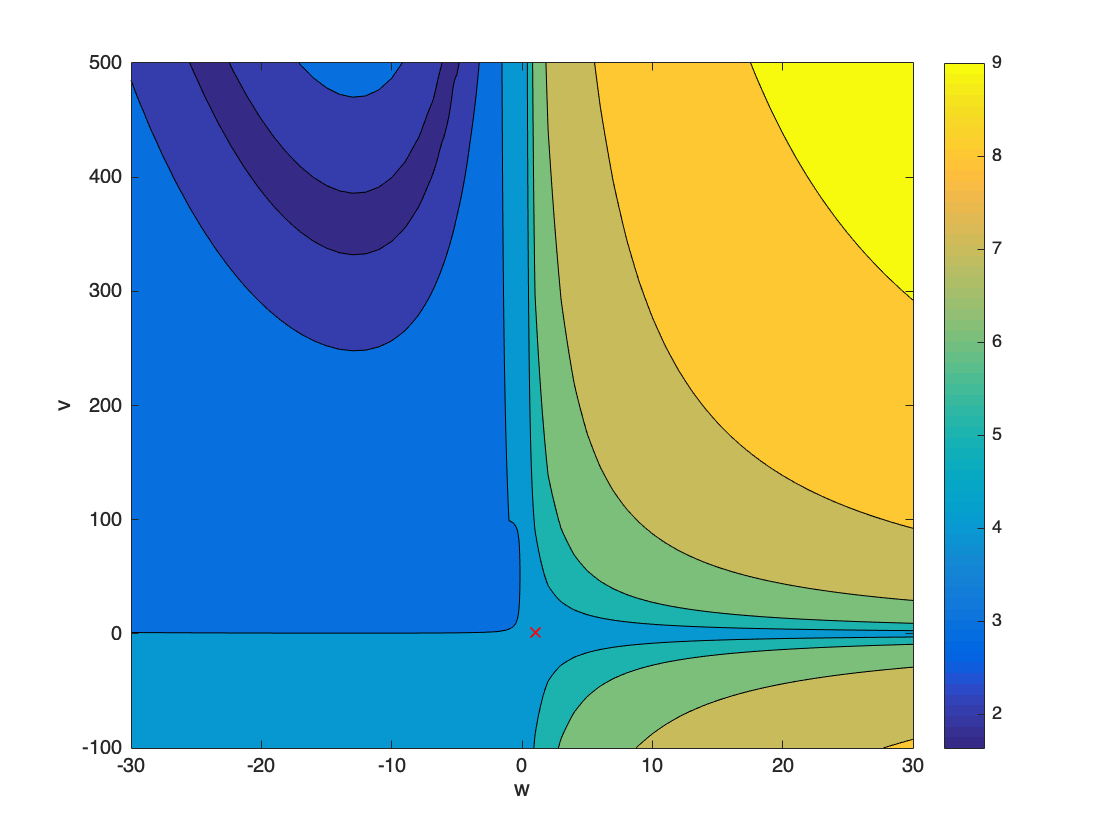}}
    \subfigure{\label{fig:swishex:subfig:c}
    \includegraphics[width=0.45\linewidth]{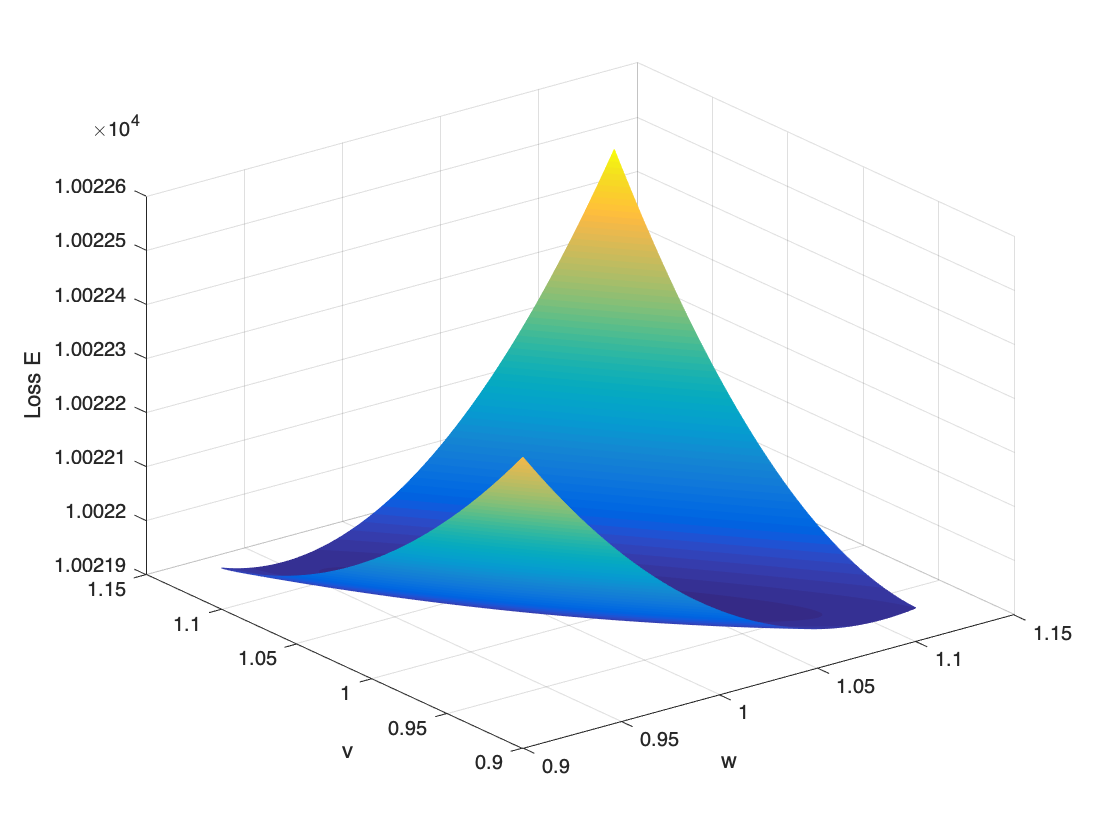}}
    \subfigure{\label{fig:swishex:subfig:d}
    \includegraphics[width=0.45\linewidth]{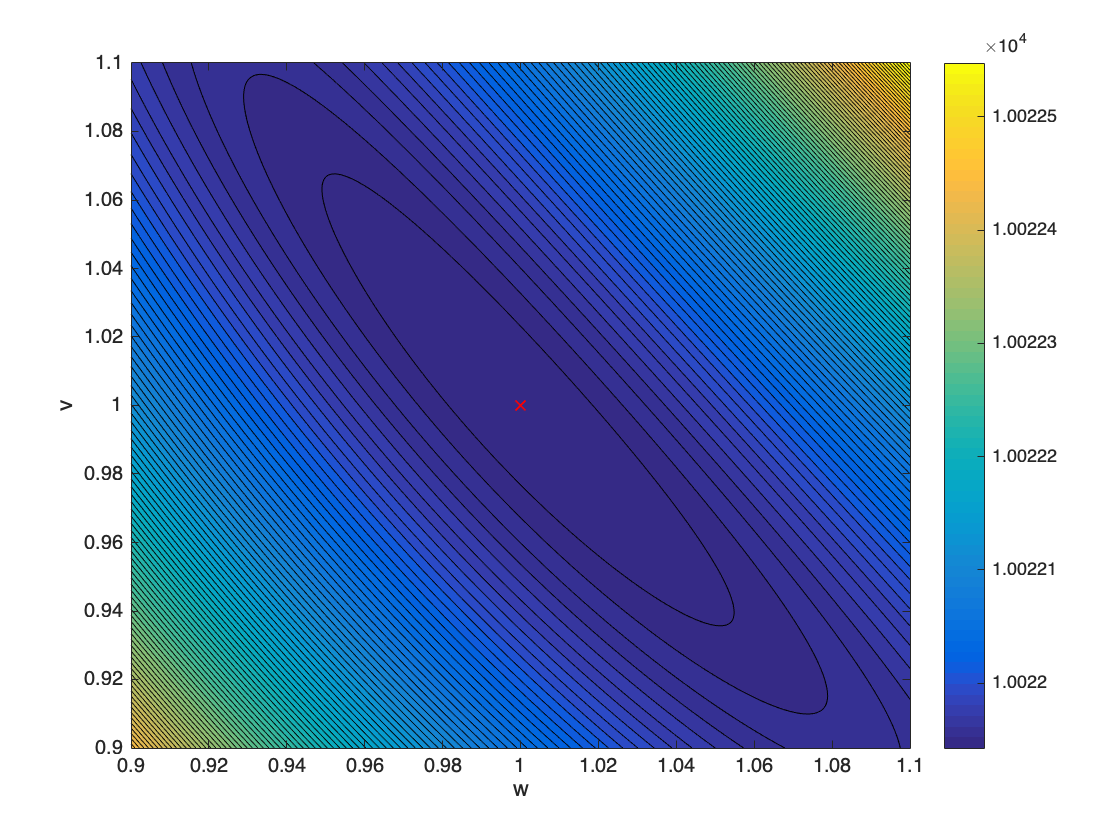}}
    \caption{The loss surface in Example \ref{ex::swish}. Figures \ref{fig:swishex:subfig:a} and \ref{fig:swishex:subfig:b} show the loss surface and the contour map in $[-30, 30]\times[-100, 500]$, respectively. The bad strict local minimum $(w, v)=(1, 1)$ is marked by a red cross. Figure \ref{fig:swishex:subfig:c} and \ref{fig:swishex:subfig:d} show the loss surface and the contour map respectively in a smaller region around the local minimum we find.}
    \label{fig:swishex}
\end{figure}

\textbf{Remark:} In both Theorem \ref{thm::bad_basin_smooth} and Example \ref{ex::swish}, we design specific output data $\mathbf{y}$ to ensure the existence of bad basins with given activation functions. If both the activation and the output data are fixed, it remains an open question whether bad basins exist for general narrow networks. For this question, \cite{christof2020stability} provides a possible direction. In particular, Corollary 5.14 of \cite{christof2020stability} claims that, under mild assumptions on the activation functions, the appearance of spurious valleys for certain choices of the training data is virtually unavoidable for deep networks if there exist vectors that are not realizable. An extension to a similar result regarding bad basins is an interesting direction to explore.

\section{Conclusions}\label{sec::conclusions}
In this paper, we studied the loss surface of wide neural networks. We showed that if the last hidden layer has no fewer neurons than the number of samples, for any continuous activation function, the loss function is weakly global, i.e., the loss landscape has no set-wise strict local minima (bad basins). We also showed that for most analytic activation functions, starting from any point, after a small perturbation there exists a strictly decreasing path to the global infimum. On the other hand, we construct examples to show that sub-optimal strict local minima can exist for a 1-hidden-layer narrow network with a large class of activation functions. Together, our results show a phase transition from narrow to wide networks in a broad setting of activations. Note that we did not establish phase transitions for all activations.

Large width does not necessarily lead to a perfect landscape for deep neural networks. Nevertheless, our work shows that large width can benefit the landscape at least by eliminating bad basins (set-wise strict local minima). It is an interesting future direction to study the other benefits of width. We hope this work can shed light on the training difficulties of small networks, which is an important and very challenging problem.

\appendix

\section{More Discussions}
\label{appen::discussions}
\textbf{Network pruning suggests that optimization is a challenge.}
We mentioned in the introduction that the findings of network pruning \cite{han2015deep} suggest that the higher error of training small networks is largely due to optimization issues, but not the lack of representation power. We elaborate below. First, network pruning papers showed that most parameters of a well-trained large network can be pruned without affecting the final training and test error too much. This implies that a small network can also fit the training data well. Second, such a small network inherits the weights from the well-trained large network, and training a small network from a random initial point often leads to significantly worse performance. In a word, the global optimal value of the small network problem is quite good, but finding the global optimum by an optimization algorithm is hard. Thus the large error of training a small network (still with a sufficient number of parameters to ensure the representation power) is mainly due to optimization issues that it is hard to find global-min. Of course if the number of parameters is very small (say, only one parameter), then the large error is due to representation power, but we are interested in a certain regime that the network has enough parameters to fit data but not enough parameters to have
good optimization properties. 

\textbf{Practical motivation: understanding small networks.}
We discussed before that one motivation of our work is to understand small networks; here we provide more discussions. Training small networks becomes a hot topic due to a few reasons. First, only small models can be implemented on small devices such as mobile phones and Internet-of-Things (IoT) devices. The so-called edge computing also calls for small models on small devices. Second, many current models require a large amount of GPUs (e.g. BigGAN \cite{brock2018large} requires 8 V100 GPUs to train 15 days), thus smaller models are more accessible to resources-constrained researchers. This problem attracted a great deal of attention from practitioners (e.g. \cite{sandler2018mobilenetv2,tan2019efficientnet,yu2018slimmable} \cite{han2015deep} \cite{frankle2019lottery}).

Unlike other practical challenges in deep learning that were quickly resolved by practitioners (e.g. training networks with more than 30 layers was a major challenge after the initial success of AlexNet in 2012, but in 2015 \cite{he2016deep} proposed ResNet that can train 1000-layer networks), training small networks remains a very challenging problem for many years. This problem is so difficult that even partial progress has led to two best paper awards of the top conference ICLR in 2015 \cite{han2015deep}  and 2019 \cite{frankle2019lottery}. Nevertheless, these two works \cite{han2015deep} \cite{frankle2019lottery} used non-conventional methods (pruning-based) to identify good solutions to the optimization problem. It is very interesting to find a way to train small networks from a random initial point. To tackle such a long-standing difficult problem, we believe understanding the possible theoretical issues will be an important first step. 

Note that although training small networks is the motivation, our work is theoretical and does not attempt to immediately provide a better training method for small networks. The slow progress on training small networks suggests that it is very difficult to solve the problem in one shot (like what ResNet did for training deep networks), and it may be better to gradually tackle the problem: first have a rather complete understanding and then try various approaches. Therefore, a fundamental theoretical analysis should be a valuable step along this direction.

\textbf{Implication of ``no bad basin''}.
One might wonder what is the implication of ``no bad basin''. Traditionally, we care about ``no bad local-min'' due to a two-step argument: first, local search algorithms like gradient descent (GD) may be stuck at local minima; second, if every local-min is a global-min, then these algorithms converge to a global-min. Notice that there is a catch in the above argument: it is not easy to rigorously prove the convergence of GD to local minima (it may converge to saddle points or do not converge at all). Nevertheless, it is still meaningful to eliminate bad local minima since they are big concerns in practice. Deep learning engineers hope to have some reassurance from theory, but they also know that it is too much to ask to have a proof for everything. In a word, even though eliminating bad local minima does not resolve all issues, it resolves one issue. Similarly, eliminating bad basins does not resolve all issues, but at least it resolves one big issue.

Many practitioners view ``basins'' as a bigger threat than non-strict local-min, since adding noise or regularizers could escape non-strict local-min but not basins \cite{wang2019learning}. In that sense, our positive result provides a potential foundation for further convergence analysis. Meanwhile, our negative result indicates that researchers should not spend too much time trying to prove global convergence for narrow networks due to the possible existence of bad basins.

In addition, our result points out a possible direction on the research of training small networks. Our result suggests that bad basins might be one challenge. How to utilize this understanding to guide the design of algorithms or networks? This is beyond the scope of this paper, and here we only provide a rough explanation. In deep learning, when researchers design algorithms or heuristics, they often have some theoretical understanding in their mind and then explain the benefit of the heuristics based on their understanding. For instance, some researchers claim that momentum is useful for training since it helps escape local minima, though without any formal argument. Similarly, if the researchers are aware of the possible existence of bad basins, then they may explain how their heuristics help eliminate, avoid or escape bad basins. Note that our current message is just the existence or non-existence of bad basins; if in the future we could provide more detailed descriptions of the basins (e.g. distribution on the loss surface, the shapes, etc.), then researchers would have more information that helps their design of heuristics or algorithms.

In summary, our work has the following implications. First, it resolves one concern of bad basins when using wide networks. Second, it suggests that further convergence analysis can be done for any wide network, no matter what activation functions, but currently we do not recommend analyzing the convergence for narrow networks. Third, it suggests that researchers interested in training small networks shall pay attention to bad basins, and they could design better methods in order to eliminate/avoid/escape bad basins. 

\section{Formal Proof of Theorem \ref{thm::continuous-pyramid}}\label{sec::pyramid-proof}

The proof of Theorem \ref{thm::continuous-pyramid} is an extension of the proof for  Theorem \ref{thm::continuous}. They are similar in most aspects, and thus we present the proof in the appendix; nevertheless, this proof requires some extra techniques.

Similar to that of Theorem \ref{thm::continuous}, the proof of Theorem \ref{thm::continuous-pyramid} consists of three steps.

\textbf{Step 1}: Prove the result for networks with specific activation functions.

We show that if the activation functions of the network satisfy some specific conditions, the resulting empirical loss function is weakly global. Assumption \ref{ass::pyramid-overall} requires the network to have a wide hidden layer, i.e., the $h_0$-th hidden layer. We first specify a class of activation functions for the ``pyramid'' structure, i.e., for layers $h_0+1 \leq h \leq H$
\begin{assumption}
\label{ass::pyramid-special-neuron}
The activation function $\sigma_h$ is continuous, strictly increasing or strictly decreasing, and its range is $\mathbb{R}$, i.e., $\sigma(\mathbb{R}) = \mathbb{R}$.  
\end{assumption}
Clearly, the activation functions satisfying Assumption \ref{ass::pyramid-special-neuron} also satisfies Assumption \ref{ass::pyramid-neuron}. We have the following theorem.

\begin{theorem}\label{thm::deep-pyramid-special}
	Given a fully connected neural network with $H$ hidden layers., activation function $\sigma_h$ for each hidden layer, and empirical loss function $E(W)=l(Y, W_{H+1}T_H)$. Suppose that Assumption \ref{ass::overall} and \ref{ass::pyramid-overall} hold, and 
	\begin{enumerate}
		\item For any $1\leq h \leq h_0$, $\sigma_h$ satisfies Assumption \ref{ass::activation};
		\item For any $h_0 +1 \leq  h \leq H$, $\sigma_h$ satisfies Assumption \ref{ass::pyramid-special-neuron}.
	\end{enumerate} 
	Then, $E(W)$ is a weakly global function.
\end{theorem}

\begin{proof}
First, we show that for any initial weight $W^o$, we can perturb it to a point $W^p$, and starting from the perturbed point $W^p$, there exists a strictly decreasing path towards the global infimum.

Let $\Theta$ denote the space of $W$ and let $W_{1:h} = (W_1, W_2, \cdots, W_h)$ denote the weights for the first $h$ hidden layers. We consider the following set:
\begin{equation}
	\Omega = \{W | \mathrm{rank}(T_{h_0}) = N, \quad \!\! \mathrm{rank}(W_h) = d_h, \quad \!\! h_0+1 \leq h \leq H+1 \}.
\end{equation}
Since $\sigma_h$ satisfies Assumption \ref{ass::activation} for any $1 \leq h \leq h_0$, according to Proposition \ref{extended prop 2}, all $W_{1:h_0}$'s that entail a non-full-rank $T_{h_0}$ only constitute a zero-measure set. That is, all $W_{1:h_0}$'s such that $T_{h_0}$ is rank-$N$ constitute a dense set. Further, note that full-rank matrices are dense. Since $d_h \leq d_{h-1}$ for all $h_0 < h \leq H+1$, $\{W_{h}|\mathrm{rank}(W_h) = d_h\}$ is dense in $\mathbb{R}^{d_h \times d_{h-1}}$. From the above analysis, we conclude that $\Omega$ is dense in $\Theta$. Therefore, for any initial weight $W^o$ and an arbitrarily small $\delta>0$, there exists $W^p = (W^p_1, W^p_2, \cdots, W^p_{H+1}) \in B(W^o, \delta)$ such that $W^p \in \Omega$.
	
In what follows, we show that if $E(W^p)> \inf_{W}(E(W))$, there exists a continuous path starting from $W^p$, and the empirical loss along which strictly decreases to $\inf_{W} E(W)$. 

\textbf{Case 1}: The global infimum of $l(Y, \cdot)$ is achievable. That is, there exists $\hat{Y}^*\in\mathbb{R}^{d_{H+1}\times N}$ such that $l(Y,\hat{Y}^*) = \inf_{\hat{Y}} l(Y, \hat{Y})$. 

Denote $T^p_h$ as the output of the $h$-th hidden layer at weight $W^p$, and $\hat{Y}^p = W^p_{H+1}T^p_{H}$ as the network output at weight $W^p$. Since the loss function $l(Y, \hat{Y})$ is convex with respect to $\hat{Y}$ (in the space $\mathbb{R}^{d_{H+1}\times N}$), it is also continuous. Suppose that $W^p$ is not a global minimum of $E(W)$, we have
\begin{equation}
l(Y, \hat{Y}^p)= E(W^p) > \inf_W E(W)\geq \inf_{\hat{Y}} l(Y, \hat{Y}).
\end{equation}
Then, there exists a continuous path $\hat{Y}(\lambda): [0,1] \rightarrow \mathbb{R}$ such that $\hat{Y}(0) = \hat{Y}^p$, $l(Y, \hat{Y}(1)) = \inf_{\hat{Y}} l(Y, \hat{Y})$, and $l(Y, Y(\lambda))$ is strictly decreasing with respect to $\lambda$, i.e.,
\begin{equation}
l(Y, Y(\lambda_1)) > l(Y, Y(\lambda_2)), \quad \!\! \forall \lambda_1 < \lambda_2, \quad \!\! \lambda_1, \lambda_2 \in [0, 1].
\end{equation}

By Assumption \ref{ass::pyramid-special-neuron}, for each $h_0+1 \leq h \leq H$, $\sigma_h$ has a continuous inverse $\sigma^{-1}_h: \mathbb{R} \rightarrow \mathbb{R}$. Also, since $W^p_h$ is of full row rank for each $h_0+1 \leq h \leq H$, it has a right inverse $(W^p_h)^\dagger$ such that $W^p_h(W^p_h)^\dagger = \mathbf{I}$. Further, as $T^p_{h_0}$ has full column rank, it has a left inverse $\left(T^p_{h_0}\right)^\dagger$ such that $\left(T^p_{h_0}\right)^\dagger \left(T^p_{h_0}\right) = \mathbf{I}$. We construct a path $W(\lambda): [0,1]\rightarrow \Theta$ as follows
\begin{equation}
W(\lambda) = (W^p_1, W^p_2, \cdots, W^p_{h_0}, W_{h_0+1}(\lambda), W^p_{h_0+2}, \cdots, W^p_{H+1})
\end{equation}
where $W_h(\lambda)$ is defined recursively as follows
\begin{subequations}
\label{eq::pyramid-path-construct}
\begin{align}
W_{h_0+1}(\lambda) &= \left[\sigma^{-1}_{h_0+1}\left(T_{h_0+1}(\lambda)\right) - \sigma_{h_0+1}^{-1}\left(T^p_{h_0+1}\right)\right]\left(T^p_{h_0}\right)^\dagger + W_{h_0+1}^p \\
T_h(\lambda) &= \left(W^p_{h+1}\right)^\dagger 
\left( 
\sigma^{-1}_{h+1}\left(T_{h+1}(\lambda)\right) - \sigma^{-1}_{h+1}\left( T^{p}_{h+1} \right) \right) + T^{p}_{h}, \nonumber \\
& \quad\quad   \quad \quad\quad \quad \quad \quad \quad \quad\quad \quad \quad\quad 
h = h_0+1, h_0+2, \cdots, H-1\\
T_H(\lambda) &= \left(W^p_{H+1}\right)^\dagger 
\left(\hat{Y}(\lambda) - \hat{Y}^p \right) + T^{p}_H.
\end{align}
\end{subequations}
Note that along the constructed path $W(\lambda)$ we only change the weight to the $h_0+1$ hidden layers. Next, we verify the following three facts.
\begin{enumerate}[label=(\arabic*)]
	\item $W(\lambda)$ is a continuous path. 
	\par In fact, since $\hat{Y}(\lambda)$ is continuous, $\sigma^{-1}_h$ is a continuous function for all $h_0+1 \leq h \leq H$, therefore each $T_h(\lambda)$ is continuous with respect $\lambda$ for all $h_0+1 \leq h \leq H$. Thus, $W(\lambda)$ is also continuous with respect to $\lambda$.
	\item $W(0) = W^p$, and $W(1)$ is a global minimum of $E(W)$.
	\par Note that $\hat{Y}(0) = \hat{Y}^p$. From \eqref{eq::pyramid-path-construct}, we have
	\begin{subequations}
		\begin{align}
			T_H(0) &= \left(W^p_{H+1}\right)^\dagger \mathbf{0} + T^p_H  = T^p_H\\
			T_h(0) &= \left(W^p_{h+1}\right)^\dagger \mathbf{0} + T^p_h = T^p_h, \quad \!\! h = h_0+1, h_0+2, \cdots, H \\
			W_{h_0+1}(0) &= \mathbf{0} \left(T^p_{h_0}\right)^\dagger + W^p_{h_0+1} = W^p_{h_0+1}.
		\end{align}
	\end{subequations}
Notice that $W(\lambda)$ is identical to $W^p$ except the weights to the $(h_0+1)$-th layer. Thus, $W(0) = W^p$.
Now consider the output of each hidden layer at weight $W(\lambda)$, denoted by $T^{\lambda}_h$. For $W(\lambda)$ and $W^p$, the weights to the first $h_0$ hidden layers are the same, and hence $T^{\lambda}_{h_0} = T^p_{h_0}$. From \eqref{eq::pyramid-path-construct}, we have
	\begin{subequations}
		\begin{align}
		T^{\lambda}_{h_0+1} =& \sigma_{h_0+1}\left(W_{h_0+1}(\lambda)T^{\lambda}_{h_0}\right) = \sigma_{h_0+1}\left(W_{h_0+1}(\lambda)T^p_{h_0}\right) \\
		=& \sigma_{h_0+1}\left(\sigma^{-1}_{h_0+1}\left(T_{h_0+1}(\lambda)\right) - \sigma_{h_0+1}^{-1}\left(T^p_{h_0+1}\right) + W^p_{h_0+1}T^p_{h_0}\right) 
		\\
		=&\sigma_{h_0+1}\Big(\sigma^{-1}_{h_0+1}\left(T_{h_0+1}(\lambda)\right) - \sigma_{h_0+1}^{-1}\left(\sigma_{h_0+1}\left(W^p_{h_0+1}T^p_{h_0}\right)\right) \nonumber \\
		&+ W^p_{h_0+1}T^p_{h_0}\Big) \\
		=&\sigma_{h_0+1}\left(\sigma^{-1}_{h_0+1}\left(T_{h_0+1}(\lambda)\right)\right) \\
		=&T_{h_0+1}(\lambda)
		\end{align}
	\end{subequations}
	For $h_0+1 < h \leq H$, if $T^{\lambda}_{h-1} = T_{h-1}(\lambda)$, we have
	\begin{subequations}
		\begin{align}
		T^{\lambda}_{h} &= \sigma_{h}\left(W^p_h T^{\lambda}_{h-1}\right) = \sigma_{h}\left(W^p_h T_{h-1}(\lambda)\right) \\
		&= \sigma_h\left(W^p_{h}\left(W^p_{h}\right)^\dagger 
		\left( 
		\sigma^{-1}_{h}\left(T_{h}(\lambda)\right) - \sigma^{-1}_{h}\left( T^{p}_{h} \right) \right) + W^p_{h}T^{p}_{h-1}\right)
		\\
		&=\sigma_{h}\left(\sigma^{-1}_h\left(T_{h}(\lambda)\right) - \sigma_{h}^{-1}\left(\sigma_h\left(W^p_{h}T^p_{h-1}\right)\right) + W^p_h T^p_{h-1}\right) \\
		&=\sigma_{h}\left(\sigma^{-1}_{h}\left(T_{h}(\lambda)\right)\right) \\
		&=T_{h}(\lambda).
		\end{align}
	\end{subequations}
	And similarly, for the network output at $W(\lambda)$, denoted by $\hat{Y}^\lambda$ we have
	\begin{subequations}
		\begin{align}
		\hat{Y}^\lambda &= W^p_{H+1}T^{\lambda}_{H} =W^p_{H+1} T_H(\lambda)\\
		&= W^p_{H+1} \left(W^p_{H+1}\right)^\dagger 
		\left(\hat{Y}(\lambda) - \hat{Y}^p \right) + W^p_{H+1}T^{p}_H\\
		&= \hat{Y}(\lambda) - \hat{Y}^p + W^p_{H+1}T^{p}_H \\
		&= \hat{Y}(\lambda).
		\end{align}
	\end{subequations}
 	Then, the empirical loss	
 	\begin{equation}
 		E(W(1)) = l(Y, \hat{Y}^1) = \inf_{\hat{Y}}l(Y, \hat{Y}) \leq \inf_W E(W).
 	\end{equation}
 	Therefore we must have $E(W(1)) = \inf_W E(W)$. That is, $E(1)$ is a global minimum of $E(W)$.
	\item $E(W(\lambda))$ is strictly decreasing with respect to $\lambda$.
	\par From (2), we have
	\begin{equation}
		E(W(\lambda)) = l(Y, \hat{Y}^{\lambda}) = l(Y, \hat{Y}(\lambda)).
	\end{equation}
	Then for any $\lambda_1, \lambda_2 \in [0,1]$ and $\lambda_1< \lambda_2$, we have
	\begin{equation}
		E(W(\lambda_1)) = l(Y, \hat{Y}(\lambda_1)) > l(Y, \hat{Y}(\lambda_2)) = E(W(\lambda_2)).
	\end{equation}
	$E(W(\lambda))$ is strictly decreasing with respect to $\lambda$.
\end{enumerate}
We conclude that $W(\lambda)$ starts from $W^p$ and is a strictly decreasing path towards the global infimum of $E(W)$.

\textbf{Case 2}: There does not exist $\hat{Y}^*\in\mathbb{R}^{d_{H+1}\times N}$ such that $l(Y,\hat{Y}^*) = \inf_{\hat{Y}} l(Y, \hat{Y})$. 

Similar to Case 1, there exists a continuous path $\hat{Y}(\lambda): [0,1) \rightarrow \mathbb{R}$ such that $\hat{Y}(0) = \hat{Y}^p$, $\lim_{\lambda \rightarrow 1} l(Y, \hat{Y}(\lambda)) = \inf_{\hat{Y}}l(Y, \hat{Y})$, and $l(Y, Y(\lambda))$ is strictly decreasing with respect to $\lambda$. We can then construct a continuous path $W(\lambda): [0,1) \rightarrow \Theta$, such that $E(W(\lambda))$ is strictly decreasing, and $\lim_{\lambda \rightarrow 1}E(W(\lambda)) = \inf_{W}E(W)$. Since the construction and analysis are identical to that in Case 1, we omit the details here.

Then we prove by contraposition that $E(W)$ is a weakly global function. Assume in contrast that there exists a sub-optimal set-wise strict local minimum of $E(W)$, denoted by $\mathcal{W}$. Note by Definition \ref{def_2}, $\mathcal{W}$ is a compact set. Let $\mathcal{W}_\delta=\{W'\mid \inf_{W\in\mathcal{W}}\|W'-W\|_2\leq\delta \}$, then there exists $\delta>0$ such that for all $W\in\mathcal{W}$ and $W'\in\mathcal{W}_\delta\setminus \mathcal{W}$, $E(W)<E(W')$. Define $\partial \mathcal{W}_\delta$ as the boundary of $\mathcal{W}_\delta$. Note that both $\mathcal{W}_\delta$ and $\partial\mathcal{W}_\delta$ are closed, there exists $W^* \in \partial \mathcal{W}_\delta$ such that $E(W^*) = \inf_{W'\in\partial \mathcal{W}_\delta}E(W')$. Moreover, $E(W^*)=\sup_{W\in\mathcal{W}}E(W)+\varepsilon$ for some $\varepsilon>0$. 
	
Consider an arbitrary point $W^o\in\mathcal{W}$. Since $E(W)$ is a continuous function, there exists $\delta>\delta_0>0$ such that for any $W'\in B(W^o, \delta_0)$, $|E(W')-E(W^o)|<\varepsilon/2$. From the above analysis, we can find $W^p \in B(W^o, \delta_0)$ such that there exists a strictly decreasing path from $W^p$ to $\inf_W E(W)$ . Since $\mathcal{W}$ is a bad local minimum, $\inf_{W\in \mathcal{W}_{\delta}} E(W) > \inf_{W} E(W)$. Therefore, the above strictly decreasing path starting from $W^p$ must pass through the boundary $\partial \mathcal{W}_\delta$. However, $E(W^p)<E(W^o)+\varepsilon/2< \sup_{W\in\mathcal{W}}E(W)+\varepsilon=E(W^*)=\inf_{W'\in\partial \mathcal{W}_\delta}E(W')$. This implies that the considered path can never be strictly decreasing, leading to a contradiction. Therefore, we conclude that there is no sub-optimal set-wise strict local minima, and therefore $E(W)$ is a weakly global function.
\end{proof}

Theorem \ref{thm::deep-pyramid-special} identifies a special class of deep over-parameterized networks with a pyramid structure, whose empirical loss function is weakly global.

\textbf{Step 2}: Show that for each hidden layer, the activation in Theorem \ref{thm::deep-pyramid-special} can approximate the activation function in Theorem \ref{thm::continuous-pyramid}. 

We use the same approximation trick as in the proof of Theorem \ref{thm::continuous}. For the activation functions satisfying Assumption \ref{ass::activation}, we have Lemma \ref{lemma::continuous_activation}. Regarding the activation functions satisfying Assumption $\ref{ass::pyramid-special-neuron}$, we have the following lemma.

\begin{lemma}\label{lemma::continuous_activation_pyramid}
For any continuous and non-increasing (or non-decreasing) function $f: \mathbb{R} \rightarrow \mathbb{R}$, there exists a sequence of functions $(f_k)_{k\in\mathbb{N}}$, all continuous and satisfying Assumption \ref{ass::pyramid-special-neuron}, such that $f_k$ converges to $f$ compactly. 
\end{lemma}

\begin{proof}
Consider an arbitrary continuous and non-increasing function $f: \mathbb{R} \rightarrow \mathbb{R}$. Note that the non-decreasing case can be proved by following the same idea, and we omit the details therein.

We construct a sequence of functions $(f_k)_{k\in \mathbb{N}}$ as
\begin{equation}
	\label{eq::pyramid-neuron-construct}
	f_k(x) = \begin{cases}
		- x - k - 1 + f(-k-1) & x < -k-1 \\ 
		f(x) - \frac{x+k+1}{(k+1)^2} & -k-1 \leq x \leq k+1 \\ 
		- x + k + 1 + f(k+1) - \frac{2}{k+1} & x > k+1
	\end{cases}
\end{equation}

First, we show that $f_k$ is continuous. To this end, we only need to verify that $f_k$ is left-continuous at $x = -k-1$ and right continuous at $x = k+1$. From \eqref{eq::pyramid-neuron-construct}, we have
\begin{subequations}
	\begin{align}
		\lim_{x\rightarrow (-k-1)^-} f_k(x)&= f(-k-1) = f_k(-k-1)\\
		\lim_{x\rightarrow (k+1)^+}f_k(x) &= f(k+1) - \frac{2}{k+1} = f_k(k+1)
	\end{align}
\end{subequations}
Thus, $f_k(x)$ is continuous for all $k \in \mathbb{N}$.

Second, we show that $f_k$ is strictly decreasing and $f(\mathbb{R}) = \mathbb{R}$. On both $(-\infty, -k-1)$ and $(k+1, +\infty)$,  $f_k$ is linear with negative slope, and hence strictly decreasing. Noting that $f_k$ is non-increasing, for any $x_1, x_2 \in [-k-1, k+1]$ and $x_1< x_2$, we have
\begin{equation}
	f_k(x_1) = f(x_1) - \frac{x_1 + k +1}{(k+1)^2} \geq f(x_2) - \frac{x_1 + k +1}{(k+1)^2} > f(x_2) - \frac{x_2 + k +1}{(k+1)^2} = f_k(x_2).
\end{equation}
Therefore $f_k$ is strictly decreasing on $[-k-1,k+1]$. As $f_k$ is continuous, we conclude that $f_k$ is strictly decreasing on $\mathbb{R}$. Further, from \eqref{eq::pyramid-neuron-construct} we have $\lim_{x\rightarrow -\infty}f_k(x) = +\infty$ and $\lim_{x\rightarrow +\infty}f_k(x) = -\infty$. Again, from the fact that $f_k$ is continuous, we have $f_k(\mathbb{R}) = \mathbb{R}$.

Finally, we show that $f_k$ converges to $f$ compactly. Consider an arbitrary compact set $S \subset \mathbb{R}$ and $\epsilon > 0$. Since $S$ is bounded, there exists $K\in \mathbb{N}$ such that $S \subset [-K-1, K+1]$. For any $k > \max\{K, \epsilon/2 - 1\}$ and $x \in S$, we have $x \in  [-k-1, k+1]$, and
\begin{equation}
	|f_k(x) - f(x)| = \left|f(x) - \frac{x+k+1}{(k+1)^2}-f(x)\right| = \frac{|x+k+1|}{(k+1)^2} \leq \frac{2}{k+1} < \epsilon. 
\end{equation}
Thus, $f_k$ converges to $f$ uniformly on $S$. As $S$ is an arbitrary compact set on $\mathbb{R}$, $f_k$ converges to $f$ compactly on $\mathbb{R}$.

We complete the proof.
\end{proof}

\textbf{Step 3}: Under compact convergence, show that the neural network considered in Theorem \ref{thm::continuous-pyramid} is weakly global.

We denote the considered network by $\mathcal{N}$. From Lemma \ref{lemma::continuous_activation}, for any $1\leq h \leq h_0$, there exists a sequence of activation functions $(\sigma_{h,k})_{k\in\mathbb{N}}$, each satisfying Assumption \ref{ass::activation}, that uniformly converges to $\sigma_h$. 
From Lemma \ref{lemma::continuous_activation_pyramid}, for any $h_0 +1 \leq h \leq H$, there exists a sequence of activation functions $(\sigma_{h,k})_{k\in\mathbb{N}}$, each satisfying Assumption \ref{ass::pyramid-special-neuron}, that compactly converges to $\sigma_h$. Since uniform convergence implies compact convergence, for all $1 \leq h \leq H$, $\sigma_{h,k}$ compactly converges to $\sigma_h$.

In the following, we show that the empirical loss of $\mathcal{N}$ can be approximated by a sequence of weakly global functions, which is identical to the analysis in the proof of Theorem \ref{thm::continuous}

For each $k \in \mathbb{N}$, we construct a neural network, denoted by $\mathcal{N}_k$, by replacing the activation function of the $h$-th hidden layer with $\sigma_{h,k}$, $h = 1, 2, \cdots, H$. For all $\mathcal{N}_k$, we assume the training dataset to be identical to that of $\mathcal{N}$. We also denote the output matrix of the $h$-th hidden layer by $T^{(k)}_h$ and the empirical loss by
\begin{equation}
E_k(W) = l\left(Y, W_{H+1}T^{(k)}_H\right).
\end{equation}
From Theorem \ref{thm::deep-pyramid-special}, $E_k$ is a weakly global function with respect to $W$, $\forall k\in \mathbb{N}$.

Consider an arbitrary compact subset $S$ in the space of $W$. For any $W\in S$, define $t^{(k)}_{h,i,n}(W) = (T^{(k)}_h)_{i,n}$ and $t_{h,i,n}(W) = (T_h)_{i,n}$ for any $k \in \mathbb{N}$, $1\leq h \leq H$, $1 \leq i \leq d_h$, and $1 \leq n \leq N$. That is, we rewrite the output of each neuron in the hidden layers as a function of $W$. We prove by induction that every sequence $(t^{(k)}_{h,i,n})_{k\in \mathbb{N}}$ converges to $t_{h,i,n}$ uniformly on $S$.

For $h =1$, we have
\begin{align}
t^{(k)}_{1,i,n}(W) &= \sigma_{1,k}\left(\sum^{d_0}_{j=1}(W_1)_{i,j}X_{j,n}\right) \\
t_{1,i,n}(W) &= \sigma_1\left(\sum^{d_0}_{j=1}(W_1)_{i,j}X_{j,n}\right).
\end{align}
Since $\sigma_{1,k}$ compactly converges to $\sigma_1$, $t^{(k)}_{1,j,n}$ uniformly converges to $t_{1,j,n}$ on $S$ for all $1\leq j\leq d_1$, $1 \leq n \leq N$.

For $h>1$, assume that $t^{(k)}_{h-1,i,n}$ uniformly converges to $t_{h-1,i,n}$ on $S$ for all $1\leq i\leq d_{h-1}$, $1 \leq n \leq N$. For the $h$-th layer, we have
\begin{align}
t^{(k)}_{h,i,n}(W) &=  \sigma_{h,k}\left(\sum^{d_{h-1}}_{j=1}(W_h)_{i,j}\left(T^{(k)}_{h-1}\right)_{j,n}\right)\nonumber\\ 
&=\sigma_{h,k}\left(\sum^{d_{h-1}}_{j=1}(W_h)_{i,j} t^{(k)}_{h-1,j,n}(W) \right) \\
t_{h,i,n}(W) &= \sigma_h\left(\sum^{d_{h-1}}_{j=1}(W_h)_{i,j}(T_{h-1})_{j,n}\right)\nonumber\\
&=\sigma_h\left(\sum^{d_{h-1}}_{j=1}(W_h)_{i,j} t_{h-1,j,n}(W) \right).
\end{align}
By the induction hypothesis, it is easy to show that $\sum^{d_{h-1}}_{j=1}(W_h)_{i,j} t^{(k)}_{h-1,j,n}(W)$ uniformly converges to $\sum^{d_{h-1}}_{j=1}(W_h)_{i,j} t_{h-1,j,n}(W)$ on $S$. Note that $(\sigma_{h,k})_{k\in\mathbb{N}}$ converges to $\sigma_h$ compactly. It directly follows from Lemma \ref{lemma::nesting} that $t^{(k)}_{h,i,n}(W)$ converges to $t_{h,i,n}(W)$.

Therefore, we conclude that $t^{(k)}_{h,i,n}$ converges to $t_{h,i,n}$ uniformly on $S$ for every $1\leq h \leq H$, $1 \leq i \leq d_h$, and $1 \leq n \leq N$.

Now we consider the empirical loss
\begin{align}
E_k(W) &= l\left(Y, W_{H+1}T^{(k)}_H\right)\\
E(W) &= l\left(Y, W_{H+1}T_H\right).
\end{align}
As every component of $T^{(k)}_H$ converges uniformly to the corresponding component of $T_H$ on $S$, it can be shown that $W_{H+1}T^{(k)}_H$ converges uniformly to $W_{H+1}T_H$ on $S$. By Lemma \ref{lemma::nesting}, where we set both $g_k$ and $g$ to the loss function $l$, we have that $E_k$ uniformly converges to $E$ on $S$. Noting that $S$ is an arbitrary compact subset in the space of $W$, the empirical loss $E_k$ converges to $E$ compactly on the space of $W$. Since $E_k(W)$ is a weakly global function for every $k\in \mathbb{N}$, by Proposition \ref{prop::compact_convergence}, $E(W)$ is also a weakly global function. We complete the proof.

\end{document}